%% file: ms.tex
\documentclass[letterpaper]{article} 
\usepackage{aaai19}  
\usepackage{times}  
\usepackage{helvet}  
\usepackage{courier}  
\usepackage[hyphens]{url}
\usepackage{graphicx}  
\usepackage{xcolor}
\usepackage{tikz}
\newcommand{\xhdr}[1]{\vspace{2mm} \noindent{\bf #1}}

\usepackage[ruled]{algorithm2e} 

\SetAlFnt{\small}
\SetAlCapFnt{\small}
\SetAlCapNameFnt{\small}
\SetAlCapHSkip{0pt}
\IncMargin{-\parindent}

\usepackage{mathtools}
\usepackage{subcaption}
\usepackage{amsthm}

\usepackage{amsmath}
\usepackage{amsfonts}
\usepackage{amssymb}

\usepackage[toc,page]{appendix}

\allowdisplaybreaks

\usepackage[utf8]{inputenc} 
\usepackage[T1]{fontenc}    
\usepackage[draft]{hyperref}       
\usepackage{url}            
\usepackage{booktabs}       
\usepackage{amsfonts}       
\usepackage{nicefrac}       
\usepackage{microtype}      

\newcommand{\citet}[1]
{\citeauthor{#1}~\shortcite{#1}}
\newcommand{\citep}{\cite}

\makeatletter
\newtheorem*{rep@theorem}{\rep@title}
\newcommand{\newreptheorem}[2]{%
\newenvironment{rep#1}[1]{%
 \def\rep@title{#2 \ref{##1}}%
 \begin{rep@theorem}}%
 {\end{rep@theorem}}}
\makeatother

\definecolor{ao(english)}{rgb}{0.0, 0.5, 0.0}\colorlet{notgreen}{blue!50!yellow}

\newreptheorem{theorem}{Theorem}
\newreptheorem{corollary}{Corollary}

\newtheorem{theorem}{Theorem}
\newtheorem{corollary}{Corollary}
\newtheorem{lemma}{Lemma}
\newtheorem{proposition}{Proposition}

\theoremstyle{definition}
\newtheorem{definition}{Definition} 

\newcommand{\Secref}[1]{Section~\ref{#1}} 
 
\newcommand{\Figref}[1]{Figure~\ref{#1}} 
\newcommand{\Appref}[1]{Appendix~\ref{#1}} 
 
\newcommand{\Eqref}[1]{Equation~(\ref{#1})}

\newcommand{\Probref}[1]{Problem~(\ref{#1})}

\newcommand{\mM}{\mathcal{M}}
\newcommand{\mU}{\mathcal{U}}

\newcommand{\POF}{\text{POF}}

\newcommand{\DeltaOPTraw}{\Delta\mathit{OPT}}
\newcommand{\DeltaOPT}[1]{\DeltaOPTraw\left({#1}\right)}
\newcommand{\MOPTsub}{\text{OPT}}
\newcommand{\MNRsub}{\text{NR}}
\newcommand{\MRsub}{\text{R}}
\newcommand{\MOPT}{M_{\MOPTsub}}
\newcommand{\MNR}{M_{\MNRsub}}
\newcommand{\MR}{M_{\MRsub}}

\DeclarePairedDelimiter\ceil{\lceil}{\rceil}
\DeclarePairedDelimiter\floor{\lfloor}{\rfloor}

\usetikzlibrary{shapes,backgrounds,calc}

\makeatletter
\tikzset{circle split part fill/.style  args={#1,#2}{%
 alias=tmp@name, 
  postaction={%
    insert path={
     \pgfextra{%
     \pgfpointdiff{\pgfpointanchor{\pgf@node@name}{center}}%
                  {\pgfpointanchor{\pgf@node@name}{east}}%
     \pgfmathsetmacro\insiderad{\pgf@x}
      \fill[#1] (\pgf@node@name.base) ([xshift=-\pgflinewidth]\pgf@node@name.east) arc
                          (0:180:\insiderad-\pgflinewidth)--cycle;
      \fill[#2] (\pgf@node@name.base) ([xshift=\pgflinewidth]\pgf@node@name.west)  arc
                           (180:360:\insiderad-\pgflinewidth)--cycle;            
         }}}}}  
 \makeatother

\newif\ifshowproofs
\showproofstrue

\title{Scalable Robust Kidney Exchange}
\author{Duncan C McElfresh\\
  Department of Mathematics\\
  University of Maryland\\
  College Park, MD 20742\\
 \texttt{dmcelfre@math.umd.edu}\\
 \And
 Hoda Bidkhori\\
 Department of Industrial Engineering\\
 University of Pittsburgh\\
 Pittsburgh, PA 15260\\
 \texttt{bidkhori@pitt.edu}\\
\And
John P Dickerson\\
  Department of Computer Science\\
  University of Maryland\\
  College Park, MD 20742\\
 \texttt{john@cs.umd.edu}\\
}

\begin{document}

\maketitle

\begin{abstract}
  \input{abstract}
\end{abstract}

\section{Introduction}\label{sec:intro}
\input{introduction}

\section{Preliminaries}\label{sec:prelims}

\input{prelims}

\section{Optimization in the Presence of Edge Weight Uncertainty}\label{sec:wt}
\input{SEC_edge_wt_uncertainty}

\section{Optimization in the Presence of Edge Existence Uncertainty}\label{sec:ex}
\input{SEC_edge_ex_uncertainty}

\section{Experimental Results}\label{sec:experiments}
\input{SEC_experiments}

\section{Robustness as Fairness}\label{sec:fairness}
\input{SEC_fairness}

\section{Conclusions \& Future Research}\label{sec:conclusions}
\input{SEC_conclusions}

{\scriptsize
\bibliographystyle{named}
\bibliography{refs,bib}
}
\clearpage
\appendix

\onecolumn
\begin{center}{\LARGE\bf Appendix to: Scalable Robust Kidney Exchange}\end{center}

\section{Edge Weight Robust Formulation}\label{app:wt}
\input{APP_wt}

\section{Edge Existence Robust Formulation}\label{app:ex}
\input{APP_ex}

\section{Robustness as Fairness}\label{app:fair}
\input{APP_fairness}

\end{document}

%% file: abstract.tex
In barter exchanges, participants directly trade their endowed goods in a constrained economic setting without money.  Transactions in barter exchanges are often facilitated via a central clearinghouse that must match participants even in the face of uncertainty---over participants, existence and quality of potential trades, and so on.  Leveraging robust combinatorial optimization techniques, we address uncertainty in kidney exchange, a real-world barter market where patients swap (in)compatible paired donors. We provide two scalable robust methods to handle two distinct types of uncertainty in kidney exchange---over the \emph{quality} and the \emph{existence} of a potential match. The latter case directly addresses a weakness in all stochastic-optimization-based methods to the kidney exchange clearing problem, which all necessarily require explicit estimates of the probability of a transaction existing---a still-unsolved problem in this nascent market.  We also propose a novel, scalable kidney exchange formulation that eliminates the need for an exponential-time constraint generation process in competing formulations, maintains provable optimality, and serves as a subsolver for our robust approach. For each type of uncertainty we demonstrate the benefits of robustness on real data from a large, fielded kidney exchange in the United States. We conclude by drawing parallels between robustness and notions of fairness in the kidney exchange setting.

%% file: introduction.tex
Real-world optimization problems face various types of uncertainty that impact both the quality and feasibility of candidate solutions. Uncertainty in combinatorial optimization is especially troublesome: if the \emph{existence} of certain constraints or variables is uncertain, identifying a good---or even feasible---solution can be extremely difficult. Stochastic optimization approaches endeavor to maximize the \emph{expected} objective value, under uncertainty. While sometimes successful, stochastic optimization relies heavily on a correct characterization of uncertainty; furthermore, stochastic approaches are often intractable---especially in combinatorial domains~\cite{Bertsimas11:Theory}. A complementary approach is \emph{robust optimization}, which protects against worst-case outcomes. Robust approaches can be less sensitive to the exact characterization of uncertainty, and are often far more tractable than stochastic approaches~\cite{BenTal09:Robust}.

This paper addresses uncertainty in \emph{kidney exchange}, a real-world barter market where patients with end-stage renal disease enter and trade their willing paired kidney donors~\cite{Rapaport86:Case,Roth04:Kidney}.  Kidney exchange is a relatively new paradigm for organ allocation, but already accounts for over 10\% of living kidney donations in the United States, and is growing in popularity worldwide~\cite{Biro17:Kidney}.  Modern exchanges also include \emph{non-directed donors} (NDDs), who enter the market without a paired patient and donate their kidney without receiving one in return.  Computationally, kidney exchange is a packing problem: solutions (matchings) consist of cyclic organ swaps and NDD-initiated donation chains in a directed compatibility graph, representing all participants and potential transactions.  Each potential transplant is given a numerical weight by policymakers; the objective is to select cycles and chains that maximize overall matching weight. In general, this problem is NP-hard~\cite{Abraham07:Clearing,Biro09:Maximum}; however, many efficient deterministic formulations exist that are fielded now and clear real exchanges~\cite{Abraham07:Clearing,Manlove15:Paired,Anderson15:Finding,Dickerson16:Position,Dickerson18:Failure}.

\xhdr{Uncertainty in kidney exchange.} Presently-fielded kidney exchange algorithms largely do not address uncertainty.  Here, we consider two types of uncertainty in kidney exchange: over the \emph{quality} of the transplant (weight uncertainty) and over the \emph{existence} of potential transplants (existence uncertainty). Policymakers assign weights to potential transplants, which are (imperfect) estimates of transplant quality; weight uncertainty stems from both measurement uncertainty (e.g., medical compatibility and kidney quality) and uncertainty in the prioritization of some patients over others. Transplant existence is always uncertain: matched transplants ``fail'' before executing for a variety of reasons, severely impacting a planned kidney exchange. To address both cases, we propose \emph{uncertainty sets} containing different realizations of the uncertain parameters. We then develop a scalable robust optimization approach, and demonstrate its success on data from a large fielded kidney exchange.

Robust optimization is a popular approach to optimization under uncertainty, with applications in reinforcement learning \cite{petrik2014raam}, regression \cite{xu2009robust}, classification \cite{chen2017robust}, and network optimization~\cite{mevissen2013data}. Motivated by real-world constraints, we apply robust optimization to kidney exchange---a graph-based market clearing or resource allocation problem.

\xhdr{Our Contributions.}
To our knowledge, weight uncertainty has not been addressed in the kidney exchange literature. Our approach is similar to that of \citet{Bertsimas04:Price} and \citet{poss2014robust}, and uses some of their results.  Several approaches have been proposed for existence uncertainty, primarily based on stochastic optimization~\cite{Dickerson16:Position,Anderson15:Finding,Dickerson18:Failure} or hierarchical optimization~\cite{Manlove15:Paired}. The primary disadvantage of these approaches---in addition to tractability---is their reliance on, and sensitivity to, the explicit estimation of the probability of each particular potential transplant.  This probability is extremely difficult to determine~\cite{Dickerson18:Failure,Glorie12:Estimating}, and prevents the translation of those methods into practice.  Our approach uses a simpler notion of edge existence uncertainty---an upper-bound on the number of non-existent edges---which is easier to interpret and estimate. \citet{glorie2014clearing} proposed a related robust formulation that is exponentially larger than ours, and is intractable for realistically-sized exchanges.

In addition, we introduce a new scalable formulation for kidney exchange that combines concepts from two state-of-the-art formulations~\cite{Anderson15:Finding,Dickerson16:Position}, handles long or uncapped NDD-initiated chains without requiring expensive constraint generation, and ties into a developed literature on fairness in kidney exchange---thus addressing use cases that are becoming more common in fielded exchanges~\cite{Anderson15:Finding}.

%% file: prelims.tex
\xhdr{Model for kidney exchange.}
A kidney exchange can be represented formally by a directed compatibility graph $G=(V,E)$.  Here, vertices $v \in V$ represent participants in the exchange, and are partitioned as $V=P\cup N$ into $P$, the set of all patient-donor pairs, and $N$, the set of all NDDs~\cite{Roth04:Kidney,Roth05:Kidney,Abraham07:Clearing}. Each potential transplant from a donor at vertex $u$ to a patient at vertex $v$ is represented by a directed edge $e = (u,v) \in E$, which has an associated weight $w_e\in \mathbf{w}$; weights are set by policymakers, and reflect both the medical utility of the transplant, as well as ethical considerations (e.g., prioritizing patients by waiting time, age, and so on). Cycles in $G$ correspond to cyclic trades between multiple patient-donor pairs in $P$; chains, correspond to donations that begin with an NDD in $N$ and continue through multiple patient-donor pairs in $P$. The kidney exchange \emph{clearing problem} is to select a feasible set of transplants (edges in $E$) that maximize overall weight. Let $\mM$ be the set of all feasible \emph{matchings} (i.e., solutions) to a kidney exchange problem; the general formulation of this problem is $\max_{\mathbf{x}\in \mM} \mathbf{x} \cdot \mathbf{w}$, where binary decision variables $\mathbf{x}$ represent edges, or cycles and chains. This problem is NP- and APX-hard~\cite{Abraham07:Clearing,Biro09:Maximum}.

\xhdr{Robust optimization.} Robust optimization is a common approach to optimization under uncertainty, which is often more tractable and requires less accurate uncertainty information than other approaches \cite{Bertsimas11:Theory}. This approach begins by defining an \emph{uncertainty set} $\mathcal{U}$ for the uncertain optimization parameter; $\mathcal{U}$ contains different \emph{realizations} of this parameter. Consider the example of edge weight uncertainty: we might design an edge weight uncertainty set $\mathcal{U}_w$ that contains the \emph{realized} (i.e. ``true'') edge weights $\hat{\mathbf{w}}$ with high probability, $P(\hat{\mathbf{w}} \in \mathcal{U}_w)\geq 1-\epsilon$, for $0<\epsilon \ll 1$. The parameter $\epsilon$ is referred to as the \emph{protection level}, and is often used to control the number of realizations in $\mathcal{U}$.

After designing $\mathcal{U}$, the robust approach finds the best solution, assuming the \emph{worst-case} realization within $\mathcal{U}$. For kidney exchange (a maximization problem), this corresponds to a \emph{minimization} over $\mathcal{U}$; for example,  \Probref{eq:robkex} is the robust formulation with uncertain edge weights.
\begin{align}\label{eq:robkex}
\max\limits_{\mathbf{x}\in \mM} \min\limits_{\hat{\mathbf{w}}\in \mU}\quad& \mathbf{x} \cdot \hat{\mathbf{w}}
\end{align}

The robustness of this approach depends on the proportion of possible realizations contained in $\mathcal{U}$. If $\mathcal{U}$ contains all possible realizations, the approach may be too conservative; if $\mathcal{U}$ only contains one possible realization of $\mathbf{\hat{w}}$, the solution may be too myopic. The number of realizations in $\mathcal{U}$ is often controlled by a parameter: either an \emph{uncertainty budget} $\Gamma$, or the protection level $\epsilon$. Next we introduce the first type of uncertainty considered in this paper: edge weight uncertainty.

%% file: SEC_edge_wt_uncertainty.tex
Edge weights in kidney exchange represent the medical and social utility gained by a \emph{single} kidney transplant. Weights are determined by policymakers, and are subject to several types of uncertainty.\footnote{The process used to set weights by the UNOS US-wide kidney exchange is published publicly~\cite{UNOS15:Revising}.} Part of this uncertainty  is due to insufficient knowledge of the future: a patient or donor's health may change, raising or lowering the ``true'' weight of their transplant edges. Another type of uncertainty stems from disagreement between policymakers regarding the social utility of a transplant. For example, some policymakers might prioritize young patients over older patients; other policymakers might prioritize the sickest patients above all healthier patients. Policymakers aggregate these value judgments to assign a single weight to each transplant edge, but this aggregation is a contentious and imperfect process (although recent work from the AI community has begun to address this using techniques from computational social choice and machine learning~\cite{Freedman18:Adapting,Noothigattu18:Making}). Still, there is no way to measure the ``true'' social utility of a transplant, and therefore this uncertainty is not easily measured. 

\xhdr{Interval weight uncertainty.} \label{sec:uwtopt}
It is beyond the scope of this work to characterize these sources of uncertainty. We simply assume that the \emph{nominal} edge weights $\mathbf{w}$, provided by policymakers, are an uncertain estimate of the \emph{realized} edge weights $\mathbf{\hat{w}}$, i.e., the ``true'' value of each transplant.  Next, we formalize edge weight uncertainty and our robust approach. This section focuses on edge weights, so we write our formulations with decision variables $x_e\in \mathbf{x}$ corresponding to individual edges.




We assume that realized edge weights $\mathbf{\hat{w}}$ are random variables with a partially known symmetric distribution, centered about the nominal weights $\mathbf{w}$. This assumption implies that $E[\mathbf{\hat{w}}]=\mathbf{w}$, thus a non-robust approach that maximizes $\mathbf{w}$ is equivalent to a stochastic optimization approach that maximizes \emph{expected} edge weight. We refer to this edge uncertainty model as \emph{interval uncertainty}.
\begin{definition}[Interval Edge Weight Uncertainty]
Let $\hat{w}_e$ be the realized weight of edge $e$, with nominal weight $w_e$, and maximum discount $0\leq d_e\leq w_e$. Let $\hat{w}_e \equiv w_e + d_e \alpha_e$, where $\alpha_e$ is the \emph{fractional deviation} of edge $e$. Both $\alpha_e$ and $\hat{w}_e$ are continuous random variables, symmetrically distributed on $[-1,1]$ and $[w_e-d_e,w_e+d_e]$ respectively.
\end{definition} 
Each discount factor $d_e$ should reflect the level of uncertainty in $w_e$. If $w_e$ is known exactly, then $d_e=0$; if $w_e$ is very uncertain, then we might set $d_e = w_e$, or higher. 

To vary the degree of uncertainty, we use an \emph{uncertainty budget} $\Gamma$, which limits the total deviation from nominal edge weights. With our uncertainty model, it is natural to let $\Gamma$ limit the total fractional deviation of each edge weight---i.e., sum of all $\alpha_e$. This uncertainty set $\mathcal{U}^{I}_\Gamma$ is defined as:

\vspace{-2mm}
{
\begin{equation*}\label{eq:usetint}
\mathcal{U}^{I}_\Gamma= \left\{\mathbf{ \hat{w}} \mid \hat{w}_e = w_e + d_e \alpha_e ,|\alpha_e|\leq 1, \sum\limits_{e\in E} | \alpha_e | \leq \Gamma \right\} 
\end{equation*}
}

For example if $\Gamma=3$, there may be three edges with $|\alpha_e|=1$, or one edge with $|\alpha_e|=1$ and four edges with $|\alpha_e|=1/2$, and so on.

Choosing an appropriate $\Gamma$ is not straightforward. Matchings often use only a small fraction of the decision variables (e.g., transplant edges), and it is difficult to predict the size of the optimal matching. Intuitively, $\Gamma$ should reflect the size of the final matching: for example if we assume that half of any matching's edges will be discounted, then we should set $\Gamma\simeq |\mathbf{x}|/2$. Generalizing this concept, we define a \emph{variable}-budget uncertainty set $\mathcal{U}^{I}_{\gamma}$, with budget function $\gamma(|\mathbf{x}|)$.

\vspace{-2mm}
{
\begin{equation*}
\mathcal{U}^{I}_{\gamma}= \left\{\mathbf{ \hat{w}} \mid \hat{w}_e = w_e + d_e \alpha_e ,|\alpha_e| \leq 1, \sum\limits_{e\in E} |\alpha_e|  \leq \gamma(|\mathbf{x}|) \right\} 
\end{equation*}
}
\vspace{-2mm}

Next, to define $\gamma$, we relate it to a much more intuitive parameter: the protection level $\epsilon$. 

\subsection{Uncertainty Budget $\gamma$ and Protection Level $\epsilon$}

The protection level $\epsilon$ mediates between a completely conservative approach, and the non-robust approach: as $\epsilon\rightarrow 0$ the approach becomes more conservative, and $\epsilon= 1$ corresponds to a non-robust approach. In this section we relate $\gamma$ to $\epsilon$, beginning with the following Theorem~\ref{thm:pboundB}.
\begin{theorem}[Adapted from Theorem 3 of \protect\cite{Bertsimas04:Price}]\label{thm:pboundB}
For a matching $\mathbf{x}\in \mM$ with $|\mathbf{x}|$ edges, and uncertainty set $\mathcal{U}^I_\Gamma$, the probability that $\mathcal{U}^I_\Gamma$ contains the realized edge weights for $\mathbf{x}$ is bounded below by
$$P(\mathbf{\hat{w}}\in \mathcal{U}^I_\Gamma) \geq 1- B(|\mathbf{x}|,\Gamma),$$
with
$$B(|\mathbf{x}|,\Gamma) =  \frac{1}{2^{|\mathbf{x}|}}\left( (1-\mu) \begin{pmatrix}|\mathbf{x}| \\ \floor{\eta} \end{pmatrix}  + \sum\limits_{l=\floor{\eta} + 1}^{|\mathbf{x}|} \begin{pmatrix}|\mathbf{x}| \\ l\end{pmatrix}\right),$$
with $\eta = (\Gamma+|\mathbf{x}|)/2$ and $\mu = \eta - \floor{\eta}$. 
\end{theorem}
That is, for some $\epsilon$, if $\Gamma$ is chosen such that $\epsilon = B(|\mathbf{x}|,\Gamma)$, then the inequality $P(\mathbf{\hat{w}}\in \mathcal{U}^I_\Gamma) \geq 1- \epsilon$ holds by Theorem \ref{thm:pboundB}. Next, we use this result to define a variable uncertainty budget function $\gamma$, using the intuitive definition introduced by \citet{poss2014robust}: for matching $\mathbf{x}\in \mM$ and protection level $\epsilon$, we find the minimum $\Gamma$ such that $B(|\mathbf{x} |,\Gamma) \leq \epsilon$. If this is not possible (i.e., the matching is too small, or $\epsilon$ is too small), then $\gamma = |\mathbf{x}|$. This budget function is defined as:

\vspace{-3mm}
{
\begin{equation*}
\beta(|\mathbf{x}|) = \begin{cases}
|\mathbf{x} | \hspace{0.3in} \text{if $\min\limits_{\Gamma >0} \left\{ \Gamma \mid B(|\mathbf{x} |,\Gamma) \leq \epsilon \right\}$ is infeasible,}\\
\min\limits_{\Gamma  >0} \left\{ \Gamma \mid B(|\mathbf{x} |,\Gamma) \leq \epsilon \right\} \hspace{0.2in} \text{otherwise.} \\ 
\end{cases}
\end{equation*}
}
It may not be clear how to solve the edge weight robust problem with this variable uncertainty budget. We use the approach of \citet{poss2014robust}, which solves the variable-budget robust problem by solving several instances of the \emph{constant}-budget robust problem; details of this approach can be found in \Appref{sec:solgamma}. Thus, to solve the variable-budget robust problem we first solve the constant-budget robust problem.

\subsection{Constant-Budget Edge Weight Robust Approach}\label{ssec:weightrobformulation}

We now describe our approach to the constant-budget edge weight robust problem; a full discussion and derivation can be found in \Appref{app:wt}. We need to solve \Probref{eq:robkex} with edge weight uncertainty set $\mathcal{U}^I_\Gamma$. This requires a minimization of the objective, over $\mathbf{\hat{w}}\in \mathcal{U}^I_\Gamma$, followed by a maximization over matchings in $\mM$. 

First we \emph{directly minimize} the objective of \Probref{eq:robkex} over $\mathcal{U}^I_\Gamma$. That is, for any matching $\mathbf{x}\in \mM$, we find the minimum objective value for any realized edge weights in $\mathcal{U}^I_\Gamma$, denoted by $Z(\mathbf{x})$:
{\small
\begin{align}\label{eq:zwt}
Z(\mathbf{x}) &= \min\limits_{\mathbf{\hat{w}}\in \mathcal{U}^I_\Gamma} \mathbf{x} \cdot \mathbf{\hat{w}}
\end{align}%
}%
%
Thus, solving the robust problem corresponds to maximizing $Z(\mathbf{x})$ over all feasible matchings. Our approach to doing so is as follows. First, we linearize $Z(\mathbf{x})$ using several new variables and constraints; we then add these to an existing kidney exchange formulation~\cite{Dickerson16:Position}. The complete linear formulations of $Z(\mathbf{x})$ and \Probref{eq:robkex} are given in \Appref{sec:wtform}, but are omitted here for space. Our robust formulation is scalable---it has a polynomial count of variables and constraints, regardless of finite chain cap; on realistic exchanges it takes only a few seconds to solve.  We demonstrate our method's impact on match composition in Section~\ref{sec:experiments}, and show how it effectively controls for the impact of robustness using protection level $\epsilon$.

%% file: SEC_edge_ex_uncertainty.tex
In this section we consider \emph{edge existence uncertainty}, where an algorithmic match must be chosen before the full realization of edges is revealed. Algorithmically-matched transplants in a kidney exchange can fail before transplantation for a variety of reasons: a patient may become too ill to undergo transplantation, or pre-transplantation testing may reveal that a patient is incompatible with her planned donor kidney. Furthermore, some edges are more likely to fail than others (e.g., edges into particularly sick patients). Edge failure significantly impacts fielded exchanges--with failure rates above 50\% in many cases~\cite{Dickerson18:Failure,Anderson15:Finding,Ashlagi13:Kidney}.

\input{fig_compatibility_graph}

For illustration, consider the simple exchange in \Figref{fig:exchange} with two potential matchings: single 5-chain initiated by the NDD, or two 2-cycles (with pairs $\{1,4\}$ and $\{2,5\}$). The 5-chain matches the most patient, but is less robust to edge failures. Consider the \emph{worst-case} outcome for each matching, when 1 edge is \emph{guaranteed} to fail: with the 5-chain, in the worst-case the \emph{first} edge fails, causing the entire chain to fail; with the 2-cycles, a single edge failure only causes a \emph{single} cycle to fail, leaving the other cycle complete. With this notion of edge existence uncertainty (which we define later), the 2-cycles are more robust than the 5-chain.

Managing edge failure in kidney exchange has been addressed in the AI and optimization literature in application-specific~\cite{Manlove15:Paired,Chen12:Graph-based} or stochastic-optimization-based~\cite{Dickerson18:Failure,Dickerson16:Position,Anderson15:Finding,Klimentova16:Maximising} ways.  These \emph{failure-aware} approaches associate with each edge a pre-determined failure probability $p_e$; these probabilities are used to then maximize \emph{expected} matching score, possibly subject to some recourse actions. This stochastic approach is tractable when $p_e$ is identical for each edge. Our work addresses two major drawbacks of the failure-aware approach. First, when each edge has a unique $p_e$, those models require enumerating every cycle and chain, which is intractable for large graphs or long chains. Second, the failure-aware approach is very sensitive to $p_e$ (as discussed in, e.g., \S4.4 of~\citet{Dickerson18:Failure}). In practice, precise values of $p_e$ are not known, thus the failure-aware approach can easily produce unreliable results. We use a simpler notion of edge existence uncertainty, which assumes that in any matching, the number of edges is \emph{bounded} by a constant ($\Gamma$). This parameter is intuitive and simple to estimate from past exchanges. 


To our knowledge, ours is the first \emph{scalable} robust optimization approach to edge existence uncertainty in kidney exchange. \citet{glorie2014clearing} develops several elegant robust methods for edge existence uncertainty, but requires that all cycles and chains are found during pre-processing and stored in memory. The number of chains grows exponentially in both the number of edges and the maximum chain length; thus, these approaches are intractable for exchanges involving more than a few dozen patient-donor pairs and NDDs.  

\xhdr{Edge existence uncertainty.} Here we briefly describe our robust approach to edge existence uncertainty; a full discussion and derivation can be found in \Appref{app:ex}. For ease of exposition, in this section, decision variables $x_c\in \mathbf{x}$ correspond to cycles and chains rather than edges. We use the following model of edge existence uncertainty.

\begin{definition}[$\Gamma$-Failures Edge Existence Uncertainty]
Up to $\Gamma$ edges may fail in any matching. After failures occur, the realized exchange graph is $\hat{G}=(V,\hat{E})$, such that edges $\hat{E}\subseteq E$ succeed and remain in existence, while all other edges fail and do not exist.
\end{definition} 
With this notion of uncertainty, without regard to computational or memory constraints, a stochastic-optimization-based approach could identify the best matching over all possible realizations $\hat{G}$~\cite{Anderson15:Finding}. This is clearly intractable, as the number of realized graphs is exponential in $|E|$.  Instead, we take a robust optimization approach by maximizing the worst-case (minimum) matching score over a set of realizable graphs $\hat{G}$ in an uncertainty set $\mathcal{U}$. Like the stochastic approach, the robust approach considers a huge number of realizations $\hat{G}$; however the robust approach is far more tractable, as it need only find the worst-case realization and need not represent all realizable graphs explicitly.

\xhdr{Uncertainty set.} Let $F\subseteq E$ be the subset of failed edges for a realized graph $\hat{G}$; thus, $\hat{E} = E\setminus F$ is the set of realized edges. \Eqref{eq:usetex} defines uncertainty set $\mathcal{U}_\Gamma^{ex}$ in this way: up to $\Gamma$ edges may fail (i.e., $|F|\leq \Gamma$). 

\vspace{-2mm}
\begin{equation}\label{eq:usetex}
\mathcal{U}^{ex}_\Gamma= \left\{ \hat{G} = (V,\hat{E}) \mid  \hat{E} = E\setminus F, |F| \leq \Gamma \right\} 
\end{equation}
\vspace{-2mm}

In kidney exchange, one edge failure can cause other edge failures: if one cycle edge fails, all edges in the cycle also fail; edge failure in a chain causes all \emph{subsequent} chain edges to also fail. This leads to a notion of weight uncertainty for cycles and chains, where the realized weight of a cycle or chain $\hat{w}_c$ may be smaller than nominal weight $w_c$. Let $\alpha_c$ be a discount parameter for cycle or chain $c$, such that  $\hat{w}_c=w_c(1-\alpha_c)$. For example, if any edge fails in cycle $c$, then the entire cycle fails and $\alpha_c=0$. We define the cycle/chain weight uncertainty set $\mathcal{U}_\Gamma^{w}$ in this way:
\begin{equation*} \label{eq:usetexwt}
\mathcal{U}_\Gamma^{w} = \left\{ \mathbf{\hat{w}_c} \mid \hat{w}_c = w_c(1-\alpha_c), \alpha_c \in [0,1] ,\sum\limits_{c\in X} \alpha_i \leq \Gamma \right\}
\end{equation*}
This uncertainty set is less intuitive than $\mathcal{U}_\Gamma^{ex}$, but more suited to the robust approach. In \Appref{app:ex} we show that $\mathcal{U}_\Gamma^{w}$ and $\mathcal{U}_\Gamma^{ex}$ are equivalent for integer $\Gamma$, and thus can be used for our robust approach.

\subsection{Robust Optimization Approach}

In this section we briefly describe our robust approach; for a full discussion and derivation, please see \Appref{app:ex}. Our robust formulation for uncertainty set $\mathcal{U}_\Gamma^{w}$ follows a similar approach to \Secref{sec:wt}. First, we directly minimize the kidney exchange objective over $\mathcal{U}^{w}_\Gamma$, for some feasible solution $\mathbf{x}\in \mM$. We express this minimization as a function $Z(\mathbf{x})$: in effect, $Z(\mathbf{x})$ discounts the $\Gamma$ largest-weight cycles and chains. We then linearize $Z(\mathbf{x})$ using several variables and constraints---this requires a formulation with variables tracking individual total chain weights---which is not possible in any existing compact kidney exchange formulations. For this purpose, we introduce a new kidney exchange formulation.

\xhdr{The PI-TSP formulation.} We propose the position-indexed TSP formulation (PI-TSP); for details, please see \Appref{app:ex}. Our formulation combines innovations from the two leading kidney exchange clearing approaches: PICEF \cite{Dickerson16:Position} and PC-TSP \cite{Anderson15:Finding}. PICEF introduced an indexing schema that enables a more compact formulation in the context of long chains; our formulation builds on this to allow tracking of individual chain weights, a necessity that PICEF could not do.  PC-TSP builds on techniques from the prize-collecting travelling salesperson problem~\cite{Balas89:Prize} to provide a tight linear programming relaxation; in general, the PC-TSP formulation has exponentially many constraints and thus requires constraint generation to solve. Our formulation uses an efficient version of position indexing that also requires only $O(|E|)+O(|V|\cdot |N|)$ constraints. Unlike PICEF, our formulation does not grow with the chain cap $L$: PICEF uses $O(|V|^3)$ variables (when $L\rightarrow|V|$); for large graphs, the PICEF model becomes too large to fit into memory~\cite{Dickerson16:Position}. Our formulation uses a fixed number of variables---$O(|V|^2)$---for any chain cap, alleviating this memory problem. This is especially relevant to existing exchanges, as long chains can significantly increase efficiency in kidney exchange~\cite{Ashlagi12:Need}. PI-TSP uses the following parameters:
\begin{itemize}
\setlength\itemsep{0em}
\item $G$: kidney exchange graph,  
\item $C$: a set of cycles on exchange graph $G$,
\item $L$: chain cap (maximum number of edges used in a chain),
\item $w_e$: edge weights for each edge $e\in E$,
\item $w^C_c$: cycle weights for each cycle $c\in C$,
\end{itemize}
and the following decision variables:
\begin{itemize}
\item $p_e \geq 1$: the position of edge $e$ in any chain,
\item $p_v \geq 1$: the position of patient-donor vertex $v$ in any chain,
\item $\hat{p}_e \geq 0$: equal to $p_e$ if $e$ is used in a chain, and $0$ otherwise,
\item $z_c \in \{0,1\}$: 1 if cycle $c$ is used in the matching, and 0 otherwise,
\item $y_e\in \{0,1\}$: 1 if edge $e$ is used in a chain, and 0 otherwise,
\item $y^n_e\in \{0,1\}$: 1 if edge $e$ is used in a chain starting with NDD $n$, and 0 otherwise,
\item $w^N_n$: total weight of the chain starting with NDD $n$,
\item $f^i_v$ and $f^o_v$: chain flow into and out of vertex $v\in P$,
\item $f^{i,n}_v$ and $f^{i,n}_v$: chain flow into and out of vertex $v\in P$, from the chain starting with NDD $n\in N$.
\end{itemize}
The PI-TSP formulation with chain cap $L$ is given in Problem \ref{eq:pitsp1}. We use the notation $\delta^-(v)$ for the set of edges into vertex $v$ and $\delta^+(v)$ for the set of edges out of $v$. 
\tiny
\begin{subequations}\label{eq:pitsp1}
\begin{align}
\text{max} \quad\sum\limits_{n\in N}w^N_{n} + \sum\limits_{c\in C} w^C_c z_c   \label{pitsp1:obj}\\
\text{s.t.} \quad \sum\limits_{e \in E} w_e y^n_e = w^N_n &\hspace{0.4in} n \in N \label{pitsp1:z}\\
\quad \sum\limits_{n \in N}y^n_e = y_e & \hspace{0.4in} e \in E \label{pitsp1:g}\\
 \quad\sum\limits_{e\in \delta^-(v)}y_e = f^i_v  &\hspace{0.4in} v\in V \label{pitsp1:b}\\
\quad\sum\limits_{e\in \delta^+(v)}y_e = f^o_v  &\hspace{0.4in} v\in V \label{pitsp1:c}\\
\quad \sum\limits_{e \in \delta^-(v)} y^n_e = f_v^{i,n} & \hspace{0.4in} v\in V, n\in N \label{pitsp1:i}\\
\quad \sum\limits_{e \in \delta^+(v)} y^n_e = f_v^{o,n} & \hspace{0.4in} v\in V, n\in N  \label{pitsp1:j}\\
\quad f^o_v + \sum\limits_{c\in C:v\in c}z_c \leq \quad f^i_v + \sum\limits_{c\in C:v\in c}z_c \leq 1 &\hspace{0.4in} v\in P \label{pitsp1:d} \\
\quad f^o_v \leq 1 & \hspace{0.4in} v\in N \label{pitsp1:e} \\
\quad p_e = 1 &\hspace{0.4in} e \in \delta^+(N) \label{pitsp1:c0}\\
\quad \hat{p}_e = p_e  y_e & \hspace{0.4in} e\in E  \label{pitsp1:c1}\\
\quad p_v = \sum\limits_{e\in \delta^-(v)} \hat{p}_e & \hspace{0.4in} v\in P \label{pitsp1:c2}\\
\quad p_e = p_v + 1 & \hspace{0.4in} v\in P,e\in \delta^+(v)  \label{pitsp1:c3}\\
\quad \sum\limits_{e \in E}y^n_e \leq L & \hspace{0.4in} n \in N \label{pitsp1:h}\\
\quad f_v^{o,n} \leq f^{i,v} \leq 1 & \hspace{0.4in} v\in V, n\in N  \label{pitsp1:k}\\
\quad y_e \in \{0,1\}& \hspace{0.4in} e\in E \\
\quad z_c \in \{ 0,1\} &\hspace{0.4in} c\in C  \\
\quad y_e^n \in \{0,1\}& \hspace{0.4in} e\in E, n \in N
\end{align}
\end{subequations}
\normalsize

The ability to express individual chain weights as decision variables has applications beyond robustness. For particularly valuable NDDs (such as those with so-called ``universal donor'' blood-type O), exchanges may enforce a \emph{minimum} chain length or chain weight, to ensure that these rare NDDs are not ``used up'' on short chains; such a policy was formerly used by the United Network for Organ Sharing~\cite{Dickerson12:Optimizing}, using a much less scalable form of optimization---that also does not consider uncertainty---than our approach~\cite{Abraham07:Clearing}. Such a policy can be implemented efficiently with PI-TSP, inefficiently with PC-TSP, and not with PICEF, where decision variables do not indicate from which NDD a chain originated. In \Appref{app:ex} we show--using real kidney exchange data--that PI-TSP can enforce a minimum chain length, and that this restriction has \emph{almost no} impact on overall matching score.

%% file: fig_compatibility_graph.tex

\tikzstyle{altruist}=[circle,
  thick,
  minimum size=1.0cm,
  draw=black!50!green!80,
  fill=black!20!green!20
]

\tikzstyle{every path}=[line width=1pt]

\begin{figure}
\centering
\begin{tikzpicture}[>=latex]

  \tikzstyle{fake}=[rectangle,minimum size=8mm,opacity=0.0]
  
  \node (alt) [shape=circle split,
    draw=gray!80,
    line width=0.6mm,text=black,font=\bfseries,
    circle split part fill={blue!20,gray!10},
    minimum size=1.2cm,
  ] at (0,0) {$n$\nodepart{lower}};

  \node (p_1) [shape=circle split,
    draw=gray!80,
    line width=0.6mm,text=black,font=\bfseries,
    circle split part fill={blue!20,red!20},
    minimum size=1.2cm,
  ] at (2,0) {$d_1$\nodepart{lower}$p_1$};

    \node (p_2) [shape=circle split,
    draw=gray!80,
    line width=0.6mm,text=black,font=\bfseries,
    circle split part fill={blue!20,red!20},
    minimum size=1.2cm,
  ] at (4,0) {$d_2$\nodepart{lower}$p_2$};

    \node (p_3) [shape=circle split,
    draw=gray!80,
    line width=0.6mm,text=black,font=\bfseries,
    circle split part fill={blue!20,red!20},
    minimum size=1.2cm,
  ] at (6,-0.75) {$d_3$\nodepart{lower}$p_3$};
  
     \node (p_4) [shape=circle split,
    draw=gray!80,
    line width=0.6mm,text=black,font=\bfseries,
    circle split part fill={blue!20,red!20},
    minimum size=1.2cm,
  ] at (1,+1.5) {$d_4$\nodepart{lower}$p_4$};
  
     \node (p_5) [shape=circle split,
    draw=gray!80,
    line width=0.6mm,text=black,font=\bfseries,
    circle split part fill={blue!20,red!20},
    minimum size=1.2cm,
  ] at (3,+1.5) {$d_5$\nodepart{lower}$p_5$};
  
     \node (p_6) [shape=circle split,
    draw=gray!80,
    line width=0.6mm,text=black,font=\bfseries,
    circle split part fill={blue!20,red!20},
    minimum size=1.2cm,
  ] at (4,-1.5) {$d_6$\nodepart{lower}$p_6$};
  
%
  
  \node (p_7) [shape=circle split,
    draw=gray!80,
    line width=0.6mm,text=black,font=\bfseries,
    circle split part fill={blue!20,red!20},
    minimum size=1.2cm,
  ] at (2,-1.5) {$d_7$\nodepart{lower}$p_7$};

  \draw[dashed,->] (alt) -- (p_1);
  \draw[dashed,->] (p_1) -- (p_2);
  \draw[dashed,->] (p_2) -- (p_3);
  \draw[dashed,->] (p_3) -- (p_6);
  \draw[dashed,->] (p_6) -- (p_7);
 
  \path[->] (p_1) edge [bend left=20] (p_4);
  \path[->] (p_4) edge [bend left=20] (p_1);
  \path[->] (p_2) edge [bend left=20] (p_5);
  \path[->] (p_5) edge [bend left=20] (p_2);

\end{tikzpicture} 
\caption{Sample exchange graph with a $5$-chain and two $2$-cycles. The NDD is denoted by $n$, and each patient (and her associated donor) is denoted by $p_i$ ($d_i$). A maximum-cardinality matching algorithm would select the $5$-chain, denoted with dashed edges; however, the smaller matching consisting of two disjoint $2$-cycles, shown with solid edges, may be more robust to edge failure.}\label{fig:exchange}
\end{figure}
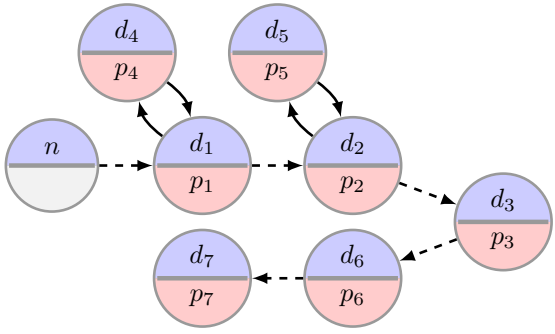

%% file: SEC_experiments.tex
In this section, we compare each robust formulation against the leading non-robust formulation, PICEF~\cite{Dickerson16:Position}, with varying levels of uncertainty. These experiments use real exchange graphs collected from the United Network for Organ Sharing (UNOS)---a large US-wide kidney exchange with over 160 participating transplant centers---between $2010$ and $2016$, as well simulated exchanges generated from known patient statistics using the standard method~\cite{Dickerson18:Failure}.\footnote{All experiments were implemented in Python and used Gurobi~\cite{Gurobi}, a state-of-the-art industrial combinatorial optimization toolkit, as a sub-solver. Our code is available on GitHub: {\tiny \texttt{\url{https://github.com/duncanmcelfresh/RobustKidneyExchange}}}.} 

For each exchange, we calculate the optimal non-robust matching $\MOPT{}$ (with total score $|\MOPT{}|$), and the robust matching $\MR{}$, for varying uncertainty budgets. We then draw many \emph{realizations} of the exchange graph, based on the uncertainty model, and calculate the realized scores of the robust matching $|\MR{}|$ and non-robust matching $|\MNR{}|$. We are primarily interested in the fractional difference from $|\MOPT{}|$, calculated as
$\DeltaOPT{M_{\{\MRsub{},\MNRsub{}\}}} = \left(|\MOPT{}|-|M_{\{\MRsub{},\MNRsub{}\}}|\right) / |\MOPT{}|. $

We calculate $\DeltaOPT{\MR{}}$ and $\DeltaOPT{\MNR{}}$ for $N=400$ realizations, and compare the robust and non-robust approaches.  In rare cases the optimal matching is empty (i.e., there is no solution, or the uncertainty budget exceeds the matching size), we exclude these exchanges from the results.

\xhdr{Edge Weight Uncertainty} We begin by exploring the impact on match utility of robust approaches to managing edge weight uncertainty.  Here, each edge is randomly labeled as \emph{probabilistic} (P) or \emph{deterministic} (D). P edges receive weight $0$ or $1$ with probability $0.5$, while D edges always receive weight $0.5$; thus, expected edge weight is always $0.5$. The non-robust approach maximizes \emph{expected} edge weight, making this a kind of stochastic approach. The robust approach considers the discount value ($0$ or $0.5$) of each edge, and avoids edges with a positive discount value. To vary the level of uncertainty, we vary the fraction of P edges ($\alpha$). Each realization is drawn by assigning the P edges to have weight either $0$ or $1$. 

We compute $\MR{}$ for protection levels $\epsilon \in \{10^{-4}, 10^{-3}, 10^{-2}, 10^{-1}, 0.5 \}$, and then calculate both $\DeltaOPT{\MR{}}$ and $\DeltaOPT{\MNR{}}$. \Figref{fig:wtresults} shows $\DeltaOPTraw{}$ on realistic $64$-vertex simulated graphs (left) and larger (typically $150$--$300$-vertex) real UNOS graphs (right); these figures show results  for each protection level $\epsilon$ and for various $\alpha$. Note that $\MNR{}$ does not depend on $\epsilon$, and thus the non-robust results are shown as (constant) dashed lines.

\begin{figure*}[ht!]
\centering
\includegraphics[width=.7\linewidth] {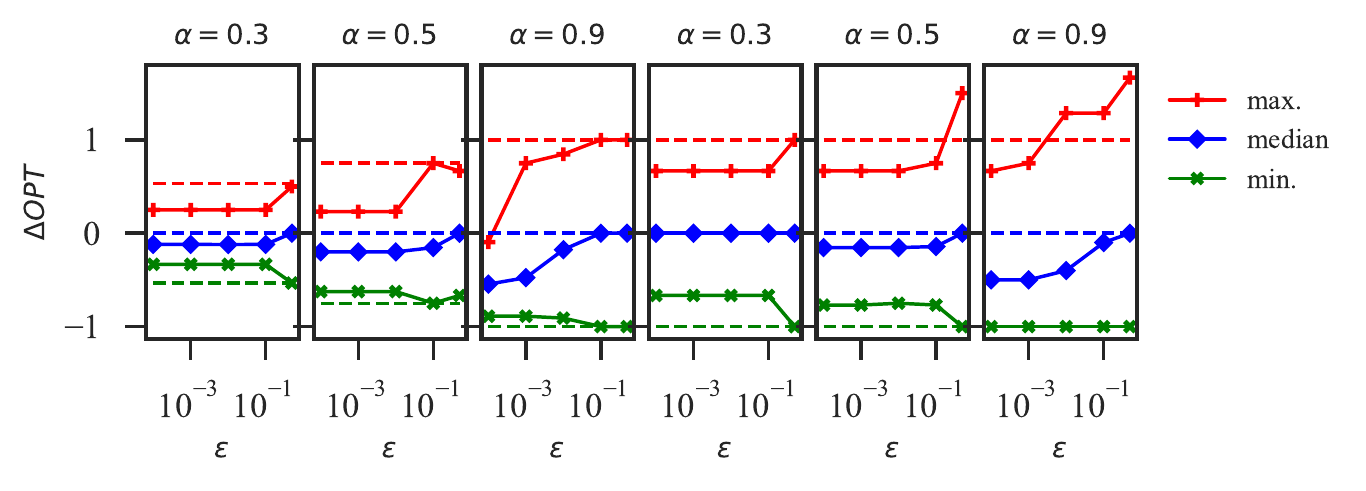}
\caption{$\DeltaOPTraw{}$ for non-robust (dashed lines) and edge weight robust (solid lines) matchings, for 64-vertex simulated exchange graphs (3 left plots) and real UNOS exchanges (3 right plots).}\label{fig:wtresults}
\end{figure*}

The robust approach guarantees a better worst-case (minimum) $\DeltaOPTraw{}$, but results in a lower median $\DeltaOPTraw{}$. The protection level $\epsilon$ controls the robustness of our approach; smaller $\epsilon$ protects against more uncertain outcomes, but at greater cost to nominal behavior. As $\epsilon\rightarrow 1$, the robust approach protects against fewer bad outcomes, and approaches the behavior of non-robust. 


\xhdr{Edge Existence Uncertainty} We now address edge existence uncertainty, and compare the robust and non-robust approaches with $\Gamma$ edge failures, for $\Gamma\in \{1,2,3,4,5 \}$. Each $\Gamma$ corresponds to a different notion of uncertainty, such that exactly $\Gamma$ edges fail.\footnote{This is slightly more conservative than the notion of uncertainty introduced previously; in \Secref{sec:ex}, \emph{up to} $\Gamma$ edges may fail, while in the experiments \emph{exactly} $\Gamma$ edges fail.}  For each $\Gamma$, we calculate $\MR{}$, and draw $N=400$ realizations by failing $\Gamma$ edges in the matching. 

We calculate $\DeltaOPTraw{}$ for each realization, and compare these results for the robust and non-robust matchings. With edge existence uncertainty, the worst-case outcome is almost always an empty matching ($\DeltaOPT{}=-1$). Thus, rather than compare the worst-case $\Delta OPT$, we compare the \emph{distribution} of $\DeltaOPTraw{}$ for each approach: we treat $\DeltaOPTraw{}$ as a random variable, and use three simple statistical tests to demonstrate that---as expected---the robust approach produces more conservative and predictable results.

First, we use the Wilcoxon signed-rank test to determine that the robust and non-robust approaches produce a different distribution of $\DeltaOPTraw{}$. For each $\Gamma$, this test produces $p$-values well below $10^{-3}$, indicating that the distributions of $\DeltaOPTraw{}$ \emph{are different} for the robust and non-robust approach. Second, for all exchanges and all $\Gamma$, the mean $\DeltaOPTraw{}$ is typically $1\%$ \emph{higher}, and the standard deviation $1$--$2\%$ \emph{lower} for the robust approach. That is, the robust approach more consistently produces higher-weight solutions.

\begin{figure*}[ht!]
\centering
\includegraphics[width=.7\linewidth] {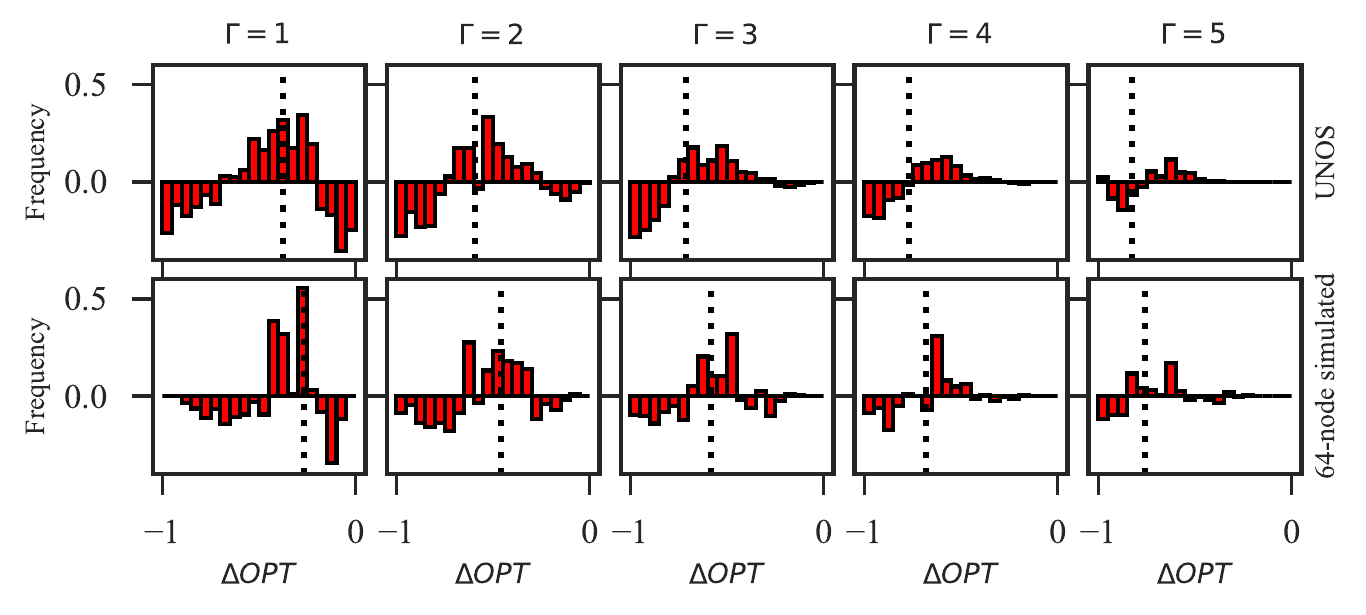}
\caption{Difference between the robust and non-robust histograms of $\DeltaOPTraw{}$ (robust minus non-robust) for real UNOS (top) and simulated exchanges (bottom), for various $\Gamma$. Dotted line: mean $\DeltaOPTraw{}$ for non-robust.}\label{fig:ex-hist}
\end{figure*}

Third, we visualize the difference between these distributions using their histograms. \Figref{fig:ex-hist} shows the bin-wise difference between the histograms of $\DeltaOPTraw{}$ (robust minus non-robust), with mean $\DeltaOPTraw{}$ for non-robust shown as a dotted line. In these plots, the height of the bars indicate the change in probability density due to robustness. On all plots, the bars are \emph{negative} for high and low values of $\DeltaOPTraw{}$, meaning that the robust matching is \emph{less likely} to have an abnormally high or low $\DeltaOPTraw{}$. The bars are \emph{positive} when $\DeltaOPTraw{}$ is near its mean non-robust value---meaning that the robust matching is \emph{more likely} to have a $\DeltaOPTraw{}$ near the mean non-robust value. This is exactly the desired behavior: robustness produces more predictable and less varied results. In this application robustness exceeds expectations: the robust approach achieves a lower variance, \emph{and} slightly improves nominal performance.

%% file: SEC_fairness.tex
Balancing efficiency and fairness is a classic economic problem; recently, a body of literature covering fairness in kidney exchange has developed in the AI/Economics~\cite{Dickerson14:Price,McElfresh18:fair,Ashlagi13:Kidney,Ding18:Non-asymptotic} and medical ethics~\cite{Gentry05:Comparison} communities; \Appref{app:fair} presents a more thorough discussion.  We now draw connections between robustness and fairness in kidney exchange.  We show that budgeted edge weight uncertainty generalizes \emph{weighted fairness} in kidney exchange, a generalization of ``priority point'' systems used in practice (see, e.g., \cite{UNOS15:Revising}). Though seemingly unrelated, fairness and robustness share a key characteristic: the balance between two competing properties. Fairness rules in kidney exchange often mediate between a fair and efficient outcome, using a parameter to set the balance. Similarly, robustness mediates between a good nominal outcome with the worst-case outcome, using an uncertainty budget or protection level to set that balance. 

In kidney exchange, fairness most often refers to the prioritization of both pediatric and \emph{highly-sensitized} patients, who are unlikely to find a match due to medical characteristics that make them incompatible with nearly all potential donors. In the weighted fairness approach, edges that represent transplants to prioritized patients receive additional edge weight, making them more likely to be matched by standard algorithms; versions of this prioritization scheme are used by most exchanges, including UNOS. To generalize weighted fairness, let each edge have a \emph{priority weight} $\hat{w}_e\in [0,\infty)$, equal to the nominal weight $w_e$ multiplied by a factor $(1+\alpha_e)$, with $\alpha_e \in [-1,\infty)$. For example, we might set $\alpha_e>0$ for all edges \emph{into} prioritized patients; this will help prioritized patients, but will likely lower overall efficiency (a tradeoff often described as the \emph{price of fairness}~\cite{Caragiannis09:Efficiency,Bertsimas11:Price,Dickerson14:Price,McElfresh18:fair}).

To balance fairness with efficiency, policymakers limit the degree of prioritization. Let $\mathcal{P}_\Gamma$ be a \emph{budgeted prioritization set}, which bounds the sum of absolute differences between each $w_e$ and $\hat{w}_e$; this prioritization set is given as:
{
\begin{equation*}
\mathcal{P}_\Gamma = \left\{ \mathbf{\hat{w}} \mid \hat{w}_e = w_e (1+\alpha_e),\alpha_e \geq -1, \sum\limits_{e\in E}\alpha_e w_e \leq \Gamma \right\}
\end{equation*}%
}%
As with edge weight uncertainty, the budget $\Gamma$ balances between fairness and efficiency. If $\Gamma$ is large, the algorithm might sacrifice matching size in order to match prioritized patients---but the maximum amount of efficiency sacrificed will be predictable, given $\Gamma$, which is attractive to policymakers. In \Appref{app:fair} we further develop this concept, propose fairness rules that use $\mathcal{P}_\Gamma$, and present some theoretical results regarding the balance between fairness and efficiency.

%% file: SEC_conclusions.tex
In this paper, we presented the first \emph{scalable} robust formulations of kidney exchange.  Our methods address both uncertainty over the \emph{quality} and the \emph{existence} of a potential transplant.  On real and simulated data from a large, fielded kidney exchange, we showed that our methods (i) clear the market within seconds and (ii) result in more predictable and better quality matchings than the status quo.


Adapting automated ethical decision-making frameworks that aggregate noisy human value judgments~\cite{Noothigattu18:Making,Freedman18:Adapting,Bonnefon16:Social} into our robust formulation is a natural way to handle uncertainty in the weights determined by a committee of stakeholders.  Approaching \emph{dynamic} kidney exchange, where participants arrive and depart over time, via robust reinforcement learning methods would be fruitful~\cite{Lim13:Reinforcement,Xu10:Distributionally}.

%% file: APP_wt.tex
We develop an edge weight robust formulation with uncertainty set $\mathcal{U}^I_\Gamma$, based on the position-indexed chain-edge formulation formulation (PICEF) introduced by \citet{Dickerson16:Position}. In \Secref{sec:picef} we review the PICEF formulation, and in \Secref{sec:wtform} we introduce our linear formulation for edge-weight robust kidney exchange. 

\Secref{sec:solGamma} and \Secref{sec:solgamma} describe the solution methods for constant uncertainty budget $\Gamma$ and variable uncertainty budget $\gamma(|\mathbf{x}|)$ for decision variables $\mathbf{x}$, respectively. 

For simplicity, we use the abbreviation $KEX(\mU)$ to refer to the robust kidney exchange problem, with uncertainty set $\mU$.

\subsection{PICEF Formulation}\label{sec:picef}

The position-indexed chain-edge formulation (PICEF) is a compact formulation proposed by \citet{Dickerson16:Position}, with a polynomial (with regard to the compatibility graph size and exogenous cycle cap) count of both variables and constraints. This formulation uses the following parameters:

\begin{itemize}
\item $G$: kidney exchange graph, consisting of edges $e\in E$ and vertices $v \in V=P \cup N$, including patient-donor pairs $P$ and NDDs $N$
\item $C$: a set of cycles on exchange graph $G$
\item $L$: chain cap (maximum number of edges used in a chain)
\item $w_e$: edge weights for each edge $e\in E$
\item $w^C_c$: cycle weights for each cycle $c\in C$, defined as $w^C_c = \sum_{e\in c} w_e$
\end{itemize}

This formulation uses one decision variable for each cycle, and several decision variables for each edge to represent chains:

\begin{itemize}
\item $z_c \in \{0,1\}$: 1 if cycle $c$ is used in the matching, and 0 otherwise
\item $y_{ek}\in \{0,1\}$: 1 if edge $e$ is used at position $k$ in a chain, and 0 otherwise
\end{itemize}

Note that edges between an NDD $n\in N$ and a patient-donor vertex $v\in P$ may only take position 1 in a chain, while edges between two patient-donor pairs may take any position $1,2,\dots, L$ in a chain. For convenience, we define the function $\mathcal{K}$ for each edge $e$, such that $\mathcal{K}(e)$ is the set of all possible positions that edge $e$ may take in a chain.

$$
\mathcal{K}(e) =
\begin{cases}
\{1\} & e \text{ begins in } n \in N \\
\{1,2,\dots,L\} & e  \text{ begins in } v\in P
\end{cases}
$$

We also use the following notation for flow into and out of vertices:

\begin{itemize}
\item $\delta^-(s)$ and $\delta^-(S)$: the set of edges into vertex $s$ or set of vertices $S$ 
\item $\delta^+(s)$ and $\delta^+(S)$: the set of edges out of vertex $s$ or set of vertices $S$ 
\end{itemize}

The PICEF formulation is given in \Probref{eq:picef}.

\begin{subequations}\label{eq:picef}
\begin{align}
\max \quad\sum\limits_{e\in E}\sum\limits_{k\in\mathcal{K}(e)} w_{e} y_{ek} + \sum\limits_{c\in C} w_c z_c   \label{picef:ob}\\
\text{s.t.} \quad\sum\limits_{e\in \delta^-(i)}\sum\limits_{k\in\mathcal{K}(e)} y_{ek} + \sum_{c\in C: i \in c} z_c\leq 1 &\hspace{0.4in} i\in P \label{picef:b}\\
\quad\sum\limits_{e\in \delta^+(i)} y_{e1} \leq 1&\hspace{0.4in} i \in N \label{picef:c} \\
\quad\sum\limits_{\begin{array}{l} e\in \delta^-(i) \wedge \\ k \in \mathcal{K}(e)\end{array} } y_{ek} \geq \sum\limits_{ e\in \delta^+(i)} y_{e,k+1}  &\hspace{0.4in}\begin{array}{l} i \in P \\ k \in \{1,\dots,L-1 \} \end{array} \label{picef:d}\\
\quad y_{ek} \in \{ 0,1\} &\hspace{0.4in} e \in E,k\in \mathcal{K}(e) \label{picef:e}\\
\quad z_c \in \{ 0,1\} &\hspace{0.4in} c\in C \label{picef:f}
\end{align}
\end{subequations}

The Objective (\ref{picef:ob}) maximizes the total weight of a matching, defined by the cycle decision variables $z_c$ and edge variables $y_{ek}$. Feasible matchings may only use each edge once, and must contain valid chains. Capacity constraints ensure that each edge is used at most once:

\noindent The capacity constraints for each vertex are as follows:
\begin{itemize}
\item Constraint \ref{picef:b}: each patient-donor vertex $i\in P$ may only participate in one cycle or one chain
\item Constraint \ref{picef:c}: each NDD  $i\in N$ may only participate in one chain
\end{itemize}

\noindent Valid chains must begin in an NDD, and conserve flow through patient-donor pairs:

\begin{itemize}
\item Constraint \ref{picef:d}: a patient-donor vertex $i \in P$ can only have an outgoing edge at position $k+1$ in a chain if it has an incoming edge at position $k$
\end{itemize}

In the next section we present the mixed integer linear program formulation for $KEX(\mathcal{U}^{I}_\Gamma)$, based on PICEF.

\subsection{Our Robust Formulation}\label{sec:wtform}

To simplify notation, let $\mM^P$ be the set of all feasible matchings for the PICEF formulation. The edge weight robust kidney exchange problem $KEX(\mathcal{U}^{I}_\Gamma)$ is given in \Eqref{eq:weightrobpicef}.

\begin{subequations}\label{eq:weightrobpicef}
\begin{align}
\max\min\limits_{\mathbf{w} \in \mathcal{U}^{I}_\Gamma} \quad\sum\limits_{e\in E}\sum\limits_{k\in\mathcal{K}(e)} w_{e} y_{ek} + \sum\limits_{c\in C} w_c z_c   \label{weightrobpicef:ob}\\
\text{s.t.} \quad (\mathbf{z},\mathbf{y})\in \mM^P
\end{align}
\end{subequations}
Proposition \ref{prop:uset} states that this problem is identical to the robust formulation with one-sided uncertainty set $\mathcal{U}^{I1}_\Gamma$---that is, $KEX(\mathcal{U}^I_\Gamma)=KEX(\mathcal{U}^{I1}_\Gamma)$.
\begin{proposition}\label{prop:uset}
The problems $KEX(\mathcal{U}^I_\Gamma)$ and $KEX(\mathcal{U}^{I1}_\Gamma)$ are equivalent.
\end{proposition}
\begin{proof}
In $KEX(\mathcal{U}^I_\Gamma)$ (Problem \ref{eq:weightrobpicef}), edge weights are minimized with respect to uncertainty set $\mathcal{U}^{I}_\Gamma$. The objective is minimized when up to $\Gamma$ edge weights are reduced by the maximum amount within $\mathcal{U}^{I}_\Gamma$ ($d_e$), and one edge weight is reduced by $(\Gamma-\floor{\Gamma})d_e$ . That is, $KEX(\mathcal{U}^I_\Gamma)$ only considers realized edge weights on the interval $\hat{w}_e \in [w_e - d_e, w_e]$. This is equivalent to restricting $\alpha_e$ to the interval $[-1,0]$ in $\mathcal{U}^{I}_\Gamma$, which is equivalent to $\mathcal{U}^{I1}_\Gamma$.
\end{proof}
Thus we must solve \Probref{eq:weightrobpicef1}, with uncertainty set $\mathcal{U}^{I1}_\Gamma$.
\begin{subequations}\label{eq:weightrobpicef1}
\begin{align}
\max\min\limits_{\mathbf{w} \in \mathcal{U}^{I1}_\Gamma} \quad\sum\limits_{e\in E}\sum\limits_{k\in\mathcal{K}(e)} w_{e} y_{ek} + \sum\limits_{c\in C} w_c z_c   \label{weightrobpicef1:ob}\\
\text{s.t.} \quad (\mathbf{z},\mathbf{y})\in \mM^P
\end{align}
\end{subequations}

Next we develop a MILP formulation for Problem \ref{eq:weightrobpicef1} by directly minimizing its Objective (\ref{weightrobpicef1:ob}). This minimum occurs when $\floor{\Gamma}$ edge weights are reduced by $d_e$, and one edge weight is reduced by $(\Gamma-\floor{\Gamma})d_e$. For this reason we refer to $d_e$ as the \emph{discount value} of edge $e$, and all edges that receive reduced weight in the robust matching are \emph{discounted}.

For simplicity, we define a variable $\hat{y}_e$ for each edge $e\in E$ such that $\hat{y}_e$ is $1$ if edge $e$ is used in the matching, and $0$ otherwise. Note that edge $e$ is used in the matching if it is used in a chain (i.e. any $y_{ek}=1$), or if it is used in a cycle (i.e. $z_c=1$ for any cycle $c$ containing $e$). Thus we define variables $\hat{y}_e$ using the following constraint.
\begin{align*}
\sum\limits_{k\in\mathcal{K}(e)}  y_{ek} + \sum\limits_{c\in C:e\in c} z_c = \hat{y}_e &\quad , e\in E \\
 \hat{y}_e \in \{0,1\}&, e\in E
\end{align*}

Next we minimize the Objective (\ref{weightrobpicef1:ob})w.r.t. $\mathcal{U}^{I1}_\Gamma$, by discounting up to $\Gamma$ edges. Note that if only $G<\Gamma$ edges are used in a matching, only $G$ edge weights may be discounted. Thus let $\Gamma'\equiv \min\{G, \Gamma \}$ be the number of discounted edges, with 
$$G=\sum\limits_{e\in E}\hat{y}_e,$$
the total number of edges used in the matching. To linearize the definition of $\Gamma'$ we introduce variable $h$, which is 1 if $G<\Gamma$ and 0 otherwise. The statement $\Gamma'= \min\left(G, \Gamma \right)$ can be linearized using the following constraints,
\begin{align*}
\Gamma - G  &\leq Wh \\
G  - \Gamma &\leq W(1-h) \\
G - Wh  &\leq \Gamma' \\
\Gamma - W(1-h) &\leq \Gamma'  \\
h &\in \{0,1\}
\end{align*}
where $W$ is a large constant.

The objective of \Probref{eq:weightrobpicef1} is minimized the the $\Gamma'$ discounted edges are those with the \emph{largest} discount value $d_e$. To select these edges we introduce variables $g_e \in\{0,1 \}$ for each edge $e\in E$. Let $m$ be the smallest $d_e$ of any discounted edge---that is, $m$ is the $\ceil{\Gamma'}^{th}$ highest $d_e$ of any edge used in the matching. We define $g_e$ as follows
\begin{equation*}
g_e= \begin{cases}
1 &\text{if } d_e \geq m \\
0 &\text{otherwise}
\end{cases}
\end{equation*}
That is, $g_e$ is 0 if $d_e$ is smaller than the $\ceil{\Gamma'}^{th}$ highest discount value of edges used in the matching, and 1 otherwise. We can define these variables using linear constraints in two steps. First, note that variables $g_e$ and $d_e$ must obey the same ordering relation. That is, $g_i \geq g_j \Leftrightarrow d_i \geq d_j$ must hold for all $i,j\in E,i\neq j$. Note that variables $d_e$ are constant, and can be sorted during pre-processing. Let $\geq_d$ indicate this ordering relation.

Next we ensure that only $\Gamma'$ edges are discounted. Note that $\hat{y}_e=1$ implies that edge $e$ is used in the matching. Edge $e$ should be discounted if it is used in the matching, and if $d_e$ is above the minimum discount value (that is, $g_e=1$). Thus, edge $e$ should be discounted if the following identity holds
$$g_e \hat{y}_e=1$$
Using this observation, we can ensure that exactly $\Gamma'$ edges are discounted with the following constraint,
$$\sum\limits_{e\in E} g_e  \hat{y}_e = \Gamma'.$$ 

For any feasible matching $M=(\mathbf{y},\mathbf{z})$, we can directly solve the minimization in \Probref{eq:weightrobpicef1} by discounting the $\Gamma'$ edges used in $M$ with the largest discount values. This is accomplished using variables $g_e$; \Eqref{eq:robsol} gives the solution of this minimization when $\Gamma$ is integer, which is expressed as a function $Z(\mathbf{y},\mathbf{z})$; the next section extends this formulation to accommodate non-integer $\Gamma$. 

\begin{subequations}\label{eq:robsol}
\begin{align}
Z( \mathbf{y},\mathbf{z})=  \quad\sum\limits_{e\in E}\sum\limits_{k\in\mathcal{K}(e)} w_{e} y_{ek} + \sum\limits_{c\in C} w_c z_c   -  \sum\limits_{e\in E}g_e d_e &\hat{y}_e\\
\text{s.t.} \quad \sum\limits_{k\in\mathcal{K}(e)}  y_{ek} + \sum\limits_{c\in C:e\in c} z_c = \hat{y}_e &, e\in E \\
\quad   \sum\limits_{e\in E}\hat{y}_e= G \\
\quad \Gamma - G  \leq W h \\
G  - \Gamma \leq W(1-h) \\
 G - W h  \leq \Gamma' \\
\quad \Gamma - W(1-h) \leq \Gamma'  \\
\quad \sum\limits_{e\in E} g_e \hat{y}_e = \Gamma'\\
\quad g_e,\hat{y}_e\in \{0,1\} &, e\in E\\
\quad g_{a} \geq_d g_{b} &, a,b\in E, a \neq b \\
\quad h \in \{0,1\}
\end{align}
\end{subequations}

%
Note that this formulation contains two sets of quadratic terms: $g_e y_{ek}$ for $k\in \mathcal{K}(e)$ for $e\in E$, and $g_e z_c$ for $c\in C$ and $e\in E$. We linearize these terms in the following section, after considering non-integer $\Gamma$.

\paragraph{Non-Integer $\Gamma$}
The number of discounted edges $\Gamma'$ may be integer or non-integer valued.  When $\Gamma'$ is not integer valued, up to $\floor{\Gamma'}$ edges are fully discounted by value $d_e$, and the edge with the smallest discount value is discounted by $(\Gamma-\floor{\Gamma})d_e$. We include this fractional discount by using two sets of indicator variables $g^f_e$ and $g^p_e$ for all $e\in E$, and then discount each edge $e$ as follows:
\begin{itemize}
\item $e$ is fully discounted if $g^p_e=g^f_e=1$.
\item $e$ is discounted by fractional amount $(\Gamma-\floor{\Gamma})$ if $g^f_e=0$ and $g^p_e=1$
\item $e$ is not discounted if $g^f_e=g^p_e=0$.
\end{itemize}
Thus if $\Gamma'$ is integer, $g^f_e=g^p_e$ for all $e\in E$; if $\Gamma'$ is not integer, then $\ceil{\Gamma'}$ matching edges should be at least partially discounted ($g^p_e=1$), and $\floor{\Gamma'}$ matching edges should be fully discounted ($g^p_e=g^f_e=1$). These indicator variables are defined in the same way as $g_e$ in \Eqref{eq:robsol}: $g^f_e,g^p_e\in \{0,1\}$, and they obey the same ordering relation as $d_e$. However, the number of matching edges with $g^f_e=1$ can be different than the number of edges with $g^p_e=1$. 

First note that $\ceil{\Gamma'}$ matching edges must have $g^p_e=1$. Recall that $G$ is the number of matching edges, and $\Gamma'=\min(\Gamma,G)$; if $\Gamma<G$, then $\ceil{\Gamma'}=\ceil{\Gamma}$, and otherwise $\ceil{\Gamma'}=G$. The variable $h$ is defined to be 1 if $G < \Gamma$ and $0$ otherwise. Thus, we use the following constraint to require that $\ceil{\Gamma'}$ matching edges have $g^p_e=1$:
$$ \sum\limits_{e\in E} g^p_e  \hat{y}_e = h G + (1-h)\ceil{\Gamma}.$$
Similarly, we can require that $\floor{\Gamma'}$ edges have $g^f_e=1$ with the following constraint
$$ \sum\limits_{e\in E} g^f_e \hat{y}_e = h G + (1-h)\floor{\Gamma}.$$
Thus if $G<\Gamma$, then all $G$ matching edges have $g^f_e=g^p_e=1$; otherwise, there are $\ceil{\Gamma}$ matching edges with $g^p_e=1$, and $\floor{\Gamma}$ matching edges with $g^f_e=1$, where the matching edge with the smallest discount has $g^f_e=0$ and $g^p_e=1$.

Using these indicator variables, the new objective of the robust formulation is
\begin{align*} \max  \sum\limits_{e\in E}\sum\limits_{k\in\mathcal{K}(e)} w_{e} y_{ek} + \sum\limits_{c\in C} w_c z_c   -  \left(1-\Gamma + \floor{\Gamma} \right)\sum\limits_{e\in E}g^f_e d_e\hat{y}_e \\
- \left(\Gamma - \floor{\Gamma} \right)\sum\limits_{e\in E}g^p_e d_e \hat{y}_e
\end{align*}
which discounts an edge $e$ by weight $d_e$ if $g^f_e=g^p_e=1$, and by weight $d_e\left(\Gamma - \floor{\Gamma} \right)$ if $g^f_e=0$ and $g^p_e=1$. Note that there are two sets of quadratic terms in this problem: $g^f_e \hat{y}_e$ and $g^p_e \hat{y}_e$ for all $e\in E$. To linearize these terms we introduce the variables $\hat{g}^f_e \equiv g^f_e \hat{y}_e$ and $\hat{g}^p_e \equiv g^p_e \hat{y}_e$, which we define using the following constraints.
\begin{align*}
\begin{array}{rl}
\hat{g}^f_e &\leq g^f_e \\
\hat{g}^f_e &\leq \hat{y}_e \\
\hat{g}^f_e &\geq g^f_e+\hat{y}_e - 1 
\end{array} 
&, e\in E\\
\hat{g}^f_e \in \{0,1\} &, e\in E \\
\\
\begin{array}{rl}
\hat{g}^p_e &\leq g^p_e \\
\hat{g}^p_e &\leq \hat{y}_e \\
\hat{g}^p_e &\geq g^p_e+\hat{y}_e - 1 
\end{array} 
&, e\in E\\
\hat{g}^p_e \in \{0,1\} &, e\in E
\end{align*}

To linearize the term $hG$, we introduce variable $\hat{g}\equiv hG$, which is defined using the following constraints. As before, $W$ is a large constant.
\begin{align*}
\hat{h} &\leq hW \\
\hat{h} &\leq G \\
\hat{h} &\geq G - (1-h)W \\
\hat{h} &\geq 0
\end{align*}

Finally, for any feasible matching $M=(\mathbf{y},\mathbf{z})$, we can directly solve the minimization in problem \ref{eq:weightrobpicef1} by discounting the $\Gamma'$ edges used in $M$ with the largest discount values. This is accomplished using variables $g^f_e$ and $g^p_e$; \Eqref{eq:robsol2} gives the solution of this minimization for general $\Gamma>0$.

%

\begin{subequations}\label{eq:robsol2}
\begin{align}
Z( \mathbf{y},\mathbf{z})=   \sum\limits_{e\in E}\sum\limits_{k\in\mathcal{K}(e)} w_{e} y_{ek} + \sum\limits_{c\in C} w_c z_c   -  \left(1-\Gamma + \floor{\Gamma} \right)&\sum\limits_{e\in E}\hat{g}^f_e d_e \\
- \left(\Gamma - \floor{\Gamma} \right)&\sum\limits_{e\in E}\hat{g}^p_e d_e \\
\text{s.t.} \quad \sum\limits_{k\in\mathcal{K}(e)}  y_{ek} + \sum\limits_{c\in C:e\in c} z_c = \hat{y}_e &, e\in E \\
\quad   \sum\limits_{e\in E}\hat{y}_e= G \\
\quad \Gamma - G  \leq W h \\
G  - \Gamma \leq W(1-h) \\
 G - W h  \leq \Gamma' \\
\quad \Gamma - W(1-h) \leq \Gamma'  \\
\quad \sum\limits_{e\in E} \hat{g}^p_e  =  \hat{h} + (1-h)\ceil{\Gamma} \\
\quad \sum\limits_{e\in E} \hat{g}^f_e  = \hat{h} + (1-h)\floor{\Gamma}\\
\quad \begin{array}{rl}
\hat{g}^f_e &\leq g^f_e \\
\hat{g}^f_e &\leq \hat{y}_e \\
\hat{g}^f_e &\geq g^f_e+\hat{y}_e - 1 
\end{array} &, e\in E\\
\quad \begin{array}{rl}
\hat{g}^p_e &\leq g^p_e \\
\hat{g}^p_e &\leq \hat{y}_e \\
\hat{g}^p_e &\geq g^p_e+\hat{y}_e - 1 
\end{array} 
&, e\in E\\
\quad \hat{h} \leq hW \\
\quad \hat{h} \leq G \\
\quad \hat{h} \geq G - (1-h)W \\
\quad g^p_e,g^f_e,\hat{y}_e\in \{0,1\} &, e\in E\\
\quad g^f_{a} \geq_d g^f_{b} &, a,b\in E, a \neq b \\
\quad g^p_{a} \geq_d g^p_{b} &, a,b\in E, a \neq b \\
\quad \hat{g}^p_e,\hat{g}^f_e \in \{0,1\} &, e\in E\\
\quad h \in \{0,1\}\\
\quad \hat{h} \geq 0
\end{align}
\end{subequations}

\Eqref{eq:robsol2} is the direct minimization of the Objective of $KEX(\mathcal{U}^{I1}_\Gamma)$ (\ref{weightrobpicef1:ob}). Thus we directly apply this minimization solution to the original \Probref{eq:weightrobpicef1}, to obtain the final linear formulation in \Eqref{eq:finalweight}.

 \begin{subequations}\label{eq:finalweight}
\begin{align}
 \quad\max  \sum\limits_{e\in E}\sum\limits_{k\in\mathcal{K}(e)} w_{e} y_{ek} + \sum\limits_{c\in C} w_c z_c   -  \left(1-\Gamma + \floor{\Gamma} \right)&\sum\limits_{e\in E}\hat{g}^f_e d_e \\
- \left(\Gamma - \floor{\Gamma} \right)&\sum\limits_{e\in E}\hat{g}^p_e d_e \\
\text{s.t.} \quad\sum\limits_{e\in \delta^-(i)}\sum\limits_{k\in\mathcal{K}(e)} y_{ek} + \sum_{c\in C: i \in c} z_c\leq 1 &\hspace{0.4in} i\in P \label{finalweight:b}\\
\quad\sum\limits_{e\in \delta^+(i)} y_{e1} \leq 1&\hspace{0.4in} i \in N \label{finalweight:c} \\
\quad\sum\limits_{\begin{array}{l} e\in \delta^-(i) \wedge \\ k \in \mathcal{K}(e)\end{array} } y_{ek} \geq \sum\limits_{ e\in \delta^+(i)} y_{e,k+1}  &\hspace{0.4in}\begin{array}{l} i \in P \\ k \in \{1,\dots,L-1 \} \end{array} \label{finalweight:d}\\
\quad \sum\limits_{k\in\mathcal{K}(e)}  y_{ek} + \sum\limits_{c\in C:e\in c} z_c = \hat{y}_e &\hspace{0.4in} e\in E \\
\quad   \sum\limits_{e\in E}\hat{y}_e= G \\
\quad \Gamma - G  \leq W h \\
G  - \Gamma \leq W(1-h) \\
 G - W h  \leq \Gamma' \\
\quad \Gamma - W(1-h) \leq \Gamma'  \\
\quad \sum\limits_{e\in E} \hat{g}^p_e  = \hat{h} + (1-h)\ceil{\Gamma} \\
\quad \sum\limits_{e\in E} \hat{g}^f_e  = \hat{h} + (1-h)\floor{\Gamma}\\
\quad \begin{array}{rl}
\hat{g}^f_e &\leq g^f_e \\
\hat{g}^f_e &\leq \hat{y}_e \\
\hat{g}^f_e &\geq g^f_e+\hat{y}_e - 1 
\end{array} &\hspace{0.4in} e\in E\\
\quad \begin{array}{rl}
\hat{g}^p_e &\leq g^p_e \\
\hat{g}^p_e &\leq \hat{y}_e \\
\hat{g}^p_e &\geq g^p_e+\hat{y}_e - 1 
\end{array} 
&\hspace{0.4in} e\in E\\
\quad \hat{h} \leq hW \\
\quad \hat{h} \leq G \\
\quad \hat{h} \geq G - (1-h)W \\
\quad g^f_{a} \geq_d g^f_{b} &\hspace{0.4in} a,b\in E, a \neq b \\
\quad g^p_{a} \geq_d g^p_{b} &\hspace{0.4in} a,b\in E, a \neq b \\
\quad y_{ek} \in \{ 0,1\} &\hspace{0.4in} e \in E,k\in \mathcal{K}(e)\\
\quad z_c \in \{ 0,1\} &\hspace{0.4in} c\in C \\
\quad g^p_e,g^f_e,\hat{y}_e\in \{0,1\} &\hspace{0.4in} e\in E\\
\quad \hat{g}^p_e,\hat{g}^f_e \in \{0,1\} &\hspace{0.4in} e\in E\\
\quad h \in \{0,1\} \\
\quad \hat{h} \geq 0
\end{align}
\end{subequations}

\subsection{Solution Method for Constant Budget $\Gamma$}\label{sec:solGamma}


This section describes the algorithm for solving the edge-weight robust formulation in \Secref{sec:wtform}, when it is unreasonable to find all cycles in the exchange graph during preprocessing. We build on the cycle pricing method in \citet{Dickerson16:Position}, which in turn built on corrected versions of methods presented by~\citet{Glorie14:Kidney} and~\citet{Plaut16:Fast}.

This method begins by solving the LP relaxation of \Probref{eq:finalweight} on a \emph{reduced model} (using a small number of cycles), and then identifying \emph{positive-price cycles}---which may improve the solution---and adding these to the model. If no positive-price cycles exist, then the solution is optimal on the reduced LP relaxation. This process is known as the \emph{pricing problem}.

After optimizing the reduced LP relaxation, we proceed in one of two ways
\begin{enumerate}
\item If the solution is fractional, then we fix one of the fractional variables and branch, as in a standard branch-and-bound tree, 
\item If the solution is integral, then it is the optimal solution to \Probref{eq:finalweight}.
\end{enumerate}

This combination of cycle pricing and branch-and-bound is known as \emph{branch-and-price}.

Algorithm \ref{alg:dfs} is the branch-and-price method for solving  \Probref{eq:finalweight}. There are only two inputs to this algorithm: the kidney exchange graph $G$, and the set of fixed decision variables $\mathbf{X}_F$. At each branch in the search tree, a new decision variable is fixed to either 0 or 1 and added to $\mathbf{X}_F$. When both 1) no positive price cycles exist for reduced model $\mathbf{M}$ and solution $\mathbf{X}$, and 2) the solution $\mathbf{X}$ is integral, then $\mathbf{X}$ is returned.

\begin{algorithm}
\SetKwInOut{Input}{input}\SetKwInOut{Output}{output}
\caption{BranchAndPrice\label{alg:dfs}} 
\Input{$G, \mathbf{X}_F$} 
\Output{Optimal Matching $X$}

 Generate subset of cycles $C'$, in $G$\;
 Create reduced model $\mathbf{M}$, with cycles $C'$\;
 $ \mathbf{X}\gets$ Solve LP relaxation of $\mathbf{M}$ \;
 $C^+ \gets \mathrm{CyclePrice}(G,\mathbf{X})$\; 
\While{$C^+ \neq \varnothing$}{
 Add cycles $C^+$ to $\mathbf{M}$\;
 $ \mathbf{X}\gets$ solve LP relaxation of $\mathbf{M}$\;
 $C^+ \gets \mathrm{CyclePrice}(G,\mathbf{X})$\;
	} 
\uIf{$\mathbf{X}$ is fractional} {
 Find fractional binary variable $X_i\in \mathbf{X}$ closest to $0.5$\;
 $\mathrm{BranchAndPrice}(G,\mathbf{X}_F \cup (X_i=0))$\;
 $\mathrm{BranchAndPrice}(G,\mathbf{X}_F \cup (X_i=1))$\;
 }
\uElse{ 
 \Return{ $\mathbf{X}$}
 } 

\end{algorithm}

The branch-and-price method in Algorithm \ref{alg:dfs} requires a cycle-pricing algorithm $\mathrm{GetCycles}$. This algorithm either returns positive-price cycles---using the reduced model $\mathbf{M}$ and the current solution to the LP relaxation, $\mathbf{X}$---or determines that none exist. We adapt the cycle-pricing algorithm usesd by \citet{Dickerson16:Position} to solve the PICEF formulation, which is based on \cite{Glorie14:Kidney} and \cite{Plaut16:Fast}. These algorithms calculate the price $p_c$ of cycle $c$ as

$$ p_c = \sum\limits_{e \in c} (w_e - \delta_{v}) $$

where $w_e$ is the weight of edge $e$ in cycle $c$, and $\delta_{e}$ is the dual value of the vertex where $e$ ends. In the edge-weight robust problem, each edge $e$ may receive its nominal weight $w_e$ or its discounted weight $(w_e - d_e)$. It is not obvious whether the nominal or discounted weights should be used during cycle pricing. 

To illustrate this problem, assume we know the optimal solution $\mathbf{X}$ to \Probref{eq:finalweight}, and the set of cycles $C$ used in $\mathbf{X}$. We consider two methods for cycle pricing.

\begin{enumerate}

\item Calculate cycle prices using  discounted edge weights $(w_e-d_e)$.

Assume that, for some cycle $c \in C$, none of the edges in $c$ are discounted in $\mathbf{X}$. During branch-and-price, it may occur that---before adding $c$ to the reduced model---the following inequalities hold

\begin{align*}
 \sum\limits_{e \in c} (w_e  - d_e- \delta_{v}) &\leq 0\\
 \sum\limits_{e \in c} (w_e - \delta_{v}) & >0
\end{align*}

If discounted edge weights are used during pricing, $c$ appears to have negative price---and will not be added to the reduced model. In this case, the calculated price is incorrectly negative, branch-and-price may return a sub-optimal solution. 

\item Calculate cycle prices using nominal edge weights $w_e$.

Assume that, for some cycle $c' \not\in C$, \emph{all} of the edges in $c'$ are discounted when it is added to the reduced model. It may occur that the following inequalities hold:

\begin{align*}
 \sum\limits_{e \in c'} (w_e  - d_e - \delta_{v}) &\leq 0\\
 \sum\limits_{e \in c'} (w_e - \delta_{v}) & >0
\end{align*}

In this case, using nominal edge weights for cycle pricing will incorrectly determine that $c'$ has a positive price, and will add $c'$ to the reduced model. 
\end{enumerate}

Neither of these methods is ideal---using discounted weights can result in a sub-optimal solution, while using nominal weights adds cycles to the reduced model. Instead, we calculate cycle prices using discounted edge weights \emph{only} for edges that will be discounted in \emph{any} matching, and nominal edge weights for all other edges. As discussed in \Secref{sec:wtform}, up to $\Gamma$ edges are discounted in every solution to \Probref{eq:finalweight}; these are the edges with the largest discount values $d_e$. For any exchange graph with $|E|$ edges, the $\min(\Gamma,|M|)$ edges with the largest discount values are \emph{always} discounted if they are used in a solution to \Probref{eq:finalweight}. Algorithm \ref{alg:price} describes this method, which uses the cycle pricer of \cite{Glorie14:Kidney} as a subroutine. Proposition \ref{prop:pricing} states that this method never incorrectly determines that a cycle has negative price---and therefore never results in a sub-optimal solution.

\begin{proposition}\label{prop:pricing}
Algorithm \ref{alg:price} never determines that a positive-price cycle has a negative price.
\end{proposition}


\begin{algorithm}
\SetKwInOut{Input}{input}\SetKwInOut{Output}{output}
\caption{CyclePrice\label{alg:price}} 
\Input{$G=(V,E),\mathbf{X}$} 
\Output{Cycle Prices}
 $d^* \gets \Gamma^{th}$ highest discount value $d_e$ in $E$\;
 $w^*_e \gets \begin{cases} w_e - d_e &\text{if  } d_e \geq d^* \\ w_e &\text{otherwise} \end{cases}$\;
 \Return{$\mathrm{PositivePriceCycles}(G,L,\mathbf{X},w^*_e)$, the cycle pricer from \cite{Glorie14:Kidney}} 
\end{algorithm}

\subsection{Solution Method for Variable Budget $\gamma$}\label{sec:solgamma}

In this section we describe a method for solving the edge-weight robust kidney exchange problem with variable budget, $KEX(\mathcal{U}^{I1}_{\gamma})$. Theorem \ref{thm:cardrestrict} is a direct adaptation of Theorem 4 of \cite{poss2014robust} to the edge-weight uncertain kidney exchange problem, which states that the solution of $KEX(\mathcal{U}^{I1}_{\gamma})$ can be found by solving several cardinality-restricted instances of $KEX(\mathcal{U}^{I1}_\Gamma)$.

\begin{theorem}\label{thm:cardrestrict}

Let $\mM$ be the set of feasible matchings, with edge decision variables $\mathbf{x}\in \mM \subset \{0,1\}^{|E|}$. The solution to $KEX(\mathcal{U}^{I1}_{\gamma})$ can be found by solving $|E|$ cardinality-restricted instances of $KEX(\mathcal{U}^{I1}_\Gamma)$, 
\begin{subequations}
\begin{align*}
\max\min\limits_{\mathbf{\hat{w}}\in \mathcal{U}^I_\Gamma}\quad& \mathbf{x} \cdot \mathbf{\hat{w}}\\
\text{s.t.}\quad &\mathbf{x}\in \mM \\
&\| \mathbf{x} \| \leq k \\
&\Gamma = \gamma(k)
\end{align*}
\end{subequations}
with $k=1,\dots,|E|$, and taking the maximum-weight solution.
\end{theorem}

The proof of this theorem is identical to the proof of Theorem 4 in \citet{poss2014robust}, and is omitted here. In practice, feasible matchings use far fewer than $|E|$ edges, and thus many fewer than $|E|$ instances of $KEX(\mathcal{U}^{I1}_\Gamma)$ must be solved. Algorithm \ref{alg:variablegamma} describes our method for solving $KEX(\mathcal{U}^{I1}_\gamma)$, which first finds the maximum cardinality matching, and then solves each cardinality-restricted problem $KEX(\mathcal{U}^{I1}_\Gamma)$.

\begin{algorithm}
\caption{EdgeWeightRobust-$\gamma$\label{alg:variablegamma}}
\SetKwInOut{Input}{input}\SetKwInOut{Output}{output}
\Input{Function $\gamma$, exchange graph $G$}
\Output{Optimal matching $\mathbf{x}$}
Find the maximum cardinality matching  $\mathbf{x}_C$\;
\For{$k \leftarrow 1$ to $\|\mathbf{x}_C\|$}{
$\Gamma \leftarrow \gamma(k)$\;
$\mathbf{x}^*_k\leftarrow $ solution to $KEX(\mathcal{U}^{I1}_\Gamma)$, restricting cardinality to $k$\;
}
\Return{The maximum-weight matching in $\{\mathbf{x}^*_k\}$}
\end{algorithm}

%% file: APP_ex.tex
In this section we develop an edge existence robust formulation for kidney exchange, using uncertainty set $\mathcal{U}_\Gamma^{w}$. Our approach is based on a formulation introduced by \citet{Anderson15:Finding}, which adapts a formulation of the prize-collecting traveling salesman problem (PC-TSP). For simplicity, we use the abbreviation $KEX(\mU)$ to refer to the robust kidney exchange problem, with uncertainty set $\mU$.

\subsection{PC-TSP Formulation}
We begin by overviewing the PC-TSP method proposed by~\citet{Anderson15:Finding}; it is based on a method for solving the prize-collecting traveling salesman problem (PC-TSP) introduced by~\citet{Balas89:Prize}. We use a version of the PC-TSP formulation with a finite chain cap; the uncapped formulation is much more compact. (Due to high failure rates, most fielded exchanges incorporate a finite maximum length of chains.  That cap can be quite high, e.g., $20$ or more, but is typically not allowed to float freely with parts of the input size, e.g., $|V|$.)  This formulation is especially useful because it allows us to define decision variables equal to each chain weight used in the matching, without explicitly enumerating all possible chains. 

This formulation uses all of the same parameters as PICEF:
\begin{itemize}
\item $G$: kidney exchange graph, consisting of edges $e\in E$ and vertices $v \in V=P \cup N$, including patient-donor pairs $P$ and NDDs $N$.
\item $C$: a set of cycles on exchange graph $G$.
\item $L$: chain cap (maximum number of edges used in a chain).
\item $w_e$: edge weights for each edge $e\in E$.
\item $w^C_c$: cycle weights for each cycle $c\in C$, defined as $w^\mathcal{C}_c = \sum_{e\in c} w_e$.
\end{itemize}

PC-TSP uses one decision variable for each cycle ($z_c$) and each edge ($y_e$), and several auxiliary decision variables that help define the constraints:

\begin{itemize}
\item $z_c \in \{0,1\}$: 1 if cycle $c$ is used in the matching, and 0 otherwise.
\item $y_e\in \{0,1\}$: 1 if edge $e$ is used in a chain, and 0 otherwise.
\item $y^n_e\in \{0,1\}$: 1 if edge $e$ is used in a chain starting with NDD $n$, and 0 otherwise.
\item $w^N_n$ (auxiliary): total weight of the chain starting with NDD $n$.
\item $f^i_v$ and $f^o_v$ (auxiliary): chain flow into and out of vertex $v\in P$, respectively.
\item $f^{i,n}_v$ and $f^{i,n}_v$ (auxiliary): chain flow into and out of vertex $v\in P$, respectively, from a chain beginning with NDD $n\in N$.
\end{itemize}

The PC-TSP formulation with chain cap $L$ is given in Problem \ref{eq:tsp}. As before, we use the notation $\delta^-(v)$ for the set of edges into vertex $v$ and $\delta^+(v)$ for the set of edges out of $v$. 

\begin{subequations}\label{eq:tsp}
\begin{align}
\text{max} \quad\sum\limits_{n\in N}w^N_{n} + \sum\limits_{c\in C} w^C_c z_c   \label{tsp:obj}\\
\text{s.t.} \quad \sum\limits_{e \in E} w_e y^n_e = w^N_n &\hspace{0.4in} n \in N \label{tsp:z}\\
\quad \sum\limits_{n \in N}y^n_e = y_e & \hspace{0.4in} e \in E \label{tsp:g}\\
 \quad\sum\limits_{e\in \delta^-(v)}y_e = f^i_v  &\hspace{0.4in} v\in V \label{tsp:b}\\
\quad\sum\limits_{e\in \delta^+(v)}y_e = f^o_v  &\hspace{0.4in} v\in V \label{tsp:c}\\
\quad \sum\limits_{e \in \delta^-(v)} y^n_e = f_v^{i,n} & \hspace{0.4in} v\in V, n\in N \label{tsp:i}\\
\quad \sum\limits_{e \in \delta^+(v)} y^n_e = f_v^{o,n} & \hspace{0.4in} v\in V, n\in N  \label{tsp:j}\\
\quad f^o_v + \sum\limits_{c\in C:v\in c}z_c \leq \quad f^i_v + \sum\limits_{c\in C:v\in c}z_c \leq 1 &\hspace{0.4in} v\in P \label{tsp:d} \\
\quad f^o_v \leq 1 & \hspace{0.4in} v\in N \label{tsp:e} \\
\quad \sum\limits_{e\in \delta^-(S)}y_e \geq f^i_v & \hspace{0.4in} S \subseteq P,v\in S \label{tsp:f} \\
\quad \sum\limits_{e \in E}y^n_e \leq L & \hspace{0.4in} n \in N \label{tsp:h}\\
\quad f_v^{o,n} \leq f^{i,v} \leq 1 & \hspace{0.4in} v\in V, n\in N  \label{tsp:k}\\
\quad y_e \in \{0,1\}& \hspace{0.4in} e\in E \\
\quad z_c \in \{ 0,1\} &\hspace{0.4in} c\in C  \\
\quad y_e^n \in \{0,1\}& \hspace{0.4in} e\in E, n \in N
\end{align}
\end{subequations}


The objective \ref{tsp:obj} maximizes the total weight of a matching, defined by the cycle decision variables $z_c$ and edge decision variables $y^e$. The auxiliary variables are defined using the following constraints:

\begin{itemize}
\item Constraint \ref{tsp:z}: defines $w^N_n$.
\item Constraint \ref{tsp:g}: defines $y_e$, using $y^n_e$.
\item Constraints \ref{tsp:b} and \ref{tsp:c}: define auxiliary variables $ f^i_v$ and $f^o_v$.
\item Constraints \ref{tsp:i} and \ref{tsp:j}: define auxiliary variables $ f^{i,n}_v$ and $f^{o,n}_v$.
\end{itemize}

There is only one capacity constraint for each patient-donor vertex and each NDD:

\begin{itemize}
\item Constraint \ref{tsp:d}: each patient-donor vertex $v$ may only be used in one cycle $c$; or, if $v$ is used in a chain, chain flow out of $v$ can only be nonzero if there is chain flow out of $v$.
\item Constraint \ref{tsp:e}: each NDD $n$ may only start one chain.
\end{itemize}

The follow constraints ensure that chain flow is conserved, and enforce the chain cap $L$:

\begin{itemize}
\item Constraint \ref{tsp:h}: chains can use no more than $L$ edges.
\item Constraint \ref{tsp:k}: chain flow out of $v$ can only be nonzero if there is chain flow out of $v$. This constraint is equivalent to \ref{tsp:d}, but for variables $f_v^{o,n}$.
\end{itemize}

The final constraints ensure that each chain includes an NDD. These are very similar to the generalized subtour elimination constraints in the TSP literature. 

\begin{itemize}
\item Constraint \ref{tsp:f}: for every subset $S$ of the donor-patient vertices, each vertex in $S$ can only participate in a chain if there is chain flow into $S$.
\end{itemize}

The number of constraints in \ref{tsp:f} grows exponentially with the number of patient-donor vertices, so it is necessary to use constraint generation with the PC-TSP formulation. We avoid constraint generation by developing a new formulation, which draws on concepts of both PC-TSP and PICEF; this formulation is introduced in the following section. 

\subsection{Our PI-TSP Formulation}

In this section we present the new position-indexed PC-TSP formulation (PI-TSP), which combines concepts from both the PC-TSP formulation and the PICEF formulation. The main advantage of our approach is in the formulation of chains. PC-TSP uses a fixed number of decision variables to allow long (or uncapped) chains, but requires constraint generation. PICEF does not require constraint generation, but the number of decision variables grows polynomially with the chain cap. 

Our approach achieves the best of both worlds: PI-TSP uses a \emph{fixed} number of decision variables for any chain cap, and does not require constraint generation. To our knowledge, ours is the first formulation to exhibit this behavior. 

PI-TSP uses the same parameters as PICEF and PC-TSP:
\begin{itemize}
\item $G$: kidney exchange graph, consisting of edges $e\in E$ and vertices $v \in V=P \cup N$, including patient-donor pairs $P$ and NDDs $N$.
\item $C$: a set of cycles on exchange graph $G$.
\item $L$: chain cap (maximum number of edges used in a chain).
\item $w_e$: edge weights for each edge $e\in E$.
\item $w^C_c$: cycle weights for each cycle $c\in C$, defined as $w^C_c = \sum_{e\in c} w_e$.
\end{itemize}

PI-TSP also uses the same decision variables (and auxiliary variables) as PC-TSP. Two additional variables are added to the formulation: $p_e,p_v\geq 1$  for each edge $e\in E$ and patient-donor vertex $v\in P$, to represent $e$ and $v$'s position in a chain.

\begin{itemize}
\item $p_e \geq 1$: the position of edge $e$ in any chain.
\item $p_v \geq 1$: the position of patient-donor vertex $v$ in any chain (equal to the position of any incoming edge).
\item $\hat{p}_e \geq 0$: equal to $p_e$ if $e$ is used in a chain, and $0$ otherwise. (i.e. $\hat{p}_e = p_e \cdot y_e$) 
\item $z_c \in \{0,1\}$: 1 if cycle $c$ is used in the matching, and 0 otherwise.
\item $y_e\in \{0,1\}$: 1 if edge $e$ is used in a chain, and 0 otherwise.
\item $y^n_e\in \{0,1\}$: 1 if edge $e$ is used in a chain starting with NDD $n$, and 0 otherwise.
\item $w^N_n$ (auxiliary): total weight of the chain starting with NDD $n$.
\item $f^i_v$ and $f^o_v$ (auxiliary): chain flow into and out of vertex $v\in P$, respectively.
\item $f^{i,n}_v$ and $f^{i,n}_v$ (auxiliary): chain flow into and out of vertex $v\in P$, respectively, from a chain beginning with NDD $n\in N$.
\end{itemize}

The PI-TSP formulation with chain cap $L$ is given in Problem \ref{eq:pitsp}. As before, we use the notation $\delta^-(v)$ for the set of edges into vertex $v$ and $\delta^+(v)$ for the set of edges out of $v$. 

\begin{subequations}\label{eq:pitsp}
\begin{align}
\text{max} \quad\sum\limits_{n\in N}w^N_{n} + \sum\limits_{c\in C} w^C_c z_c   \label{pitsp:obj}\\
\text{s.t.} \quad \sum\limits_{e \in E} w_e y^n_e = w^N_n &\hspace{0.4in} n \in N \label{pitsp:z}\\
\quad \sum\limits_{n \in N}y^n_e = y_e & \hspace{0.4in} e \in E \label{pitsp:g}\\
 \quad\sum\limits_{e\in \delta^-(v)}y_e = f^i_v  &\hspace{0.4in} v\in V \label{pitsp:b}\\
\quad\sum\limits_{e\in \delta^+(v)}y_e = f^o_v  &\hspace{0.4in} v\in V \label{pitsp:c}\\
\quad \sum\limits_{e \in \delta^-(v)} y^n_e = f_v^{i,n} & \hspace{0.4in} v\in V, n\in N \label{pitsp:i}\\
\quad \sum\limits_{e \in \delta^+(v)} y^n_e = f_v^{o,n} & \hspace{0.4in} v\in V, n\in N  \label{pitsp:j}\\
\quad f^o_v + \sum\limits_{c\in C:v\in c}z_c \leq \quad f^i_v + \sum\limits_{c\in C:v\in c}z_c \leq 1 &\hspace{0.4in} v\in P \label{pitsp:d} \\
\quad f^o_v \leq 1 & \hspace{0.4in} v\in N \label{pitsp:e} \\
\quad p_e = 1 &\hspace{0.4in} e \in \delta^+(N) \label{pitsp:c0}\\
\quad \hat{p}_e = p_e  y_e & \hspace{0.4in} e\in E  \label{pitsp:c1}\\
\quad p_v = \sum\limits_{e\in \delta^-(v)} \hat{p}_e & \hspace{0.4in} v\in P \label{pitsp:c2}\\
\quad p_e = p_v + 1 & \hspace{0.4in} v\in P,e\in \delta^+(v)  \label{pitsp:c3}\\
\quad \sum\limits_{e \in E}y^n_e \leq L & \hspace{0.4in} n \in N \label{pitsp:h}\\
\quad f_v^{o,n} \leq f^{i,v} \leq 1 & \hspace{0.4in} v\in V, n\in N  \label{pitsp:k}\\
\quad y_e \in \{0,1\}& \hspace{0.4in} e\in E \\
\quad z_c \in \{ 0,1\} &\hspace{0.4in} c\in C  \\
\quad y_e^n \in \{0,1\}& \hspace{0.4in} e\in E, n \in N
\end{align}
\end{subequations}

All constraints are identical to those of PC-TSP, but without the subtour elimination constraints \ref{tsp:f}, and with the addition of the following constraints:
\begin{itemize}
\item Constraints \ref{pitsp:c0}: sets $p_e=1$ for all edges out of NDD vertices.
\item Constraints \ref{pitsp:c1}: defines $\hat{p}_e$.
\item Constraints \ref{pitsp:c2}: for all vertices $v$, sets $p_v$ equal to the variable $p_e$ of \emph{any} incoming edge.
\item Constraints \ref{pitsp:c3}: for all outgoing edges of all vertices $v$, sets $p_e=p_v+1$.
\end{itemize}

Two adjustments may be made to this formulation: first, the variables $p_v$ are not necessary, but are useful for illustration. We can remove these variables by combining Constraints \ref{pitsp:c2} and \ref{pitsp:c3} as follows:
\begin{align*}
 p_{\overline{e}} = 1 + \sum\limits_{e\in \delta^-(v)} \hat{p}_e \hspace{0.5in}  v\in P,\overline{e}\in \delta^+(v) 
\end{align*}
Second, Constraints \ref{pitsp:c1} are nonlinear; we linearize these by replacing \ref{pitsp:c1} with the following constraints for each $e\in E$:
\begin{align*}
 \hat{p}_e &\leq y_e M \\
 \hat{p}_e& \leq p_e \\
 p_e - (1-y_e) M &\leq \hat{p}_e
\end{align*}

\subsubsection{Experiments: Minimum Chain Length} 

We demonstrate the utility of the PI-TSP formulation by finding optimal matchings with a \emph{minimum} chain length ($L_{min}$). We set a \emph{maximum} chain length of $L_{max}=3$, and vary the $L_{min}$ from 0 to 3. For some exchange graph, let $|M_{OPT}|$ be the score of the \emph{optimal} matching (i.e. with no minimum chain length, and maximum chain length 3); we calculate the fractional optimality gap for the matching $M_l$ (with score $|M_l|$), which has minimum chain length $L_{min}=l$. We define $\Delta OPT(M_l)$ as
$$ \Delta OPT (M_l) = \frac{|M_l| - |M_{OPT}|}{|M_{OPT}|}$$
We calculate optimal matchings for $L_{min}=0,1,2,3$, for each of the UNOS exchange graphs used in \Secref{sec:experiments}. Only 154 of the roughly 300 UNOS graphs contain chains; the remaining graphs may have no NDDs, or the NDDs may have no feasible donors. Focusing on these 154 graphs, we calculate $\Delta OPT$ and the chain lengths of each optimal matching, for each $L_{min}$.  \Figref{fig:chain} shows histograms of $\Delta OPT$ and the chain lengths for all optimal matchings, for each $L_{min}=0,1,2,3$. Note that $\Delta OPT$ is zero for $L_{min}=0$, by definition.

For some of these exchanges, a minimum chain length of 2 or 3 was infeasible (58 for $L_{min}=2$, and 77 for $L_{min}=3$, out of 154 total exchanges); we do not consider these cases.

\begin{figure*}[ht!]
\centering
\includegraphics[width=.95\linewidth] {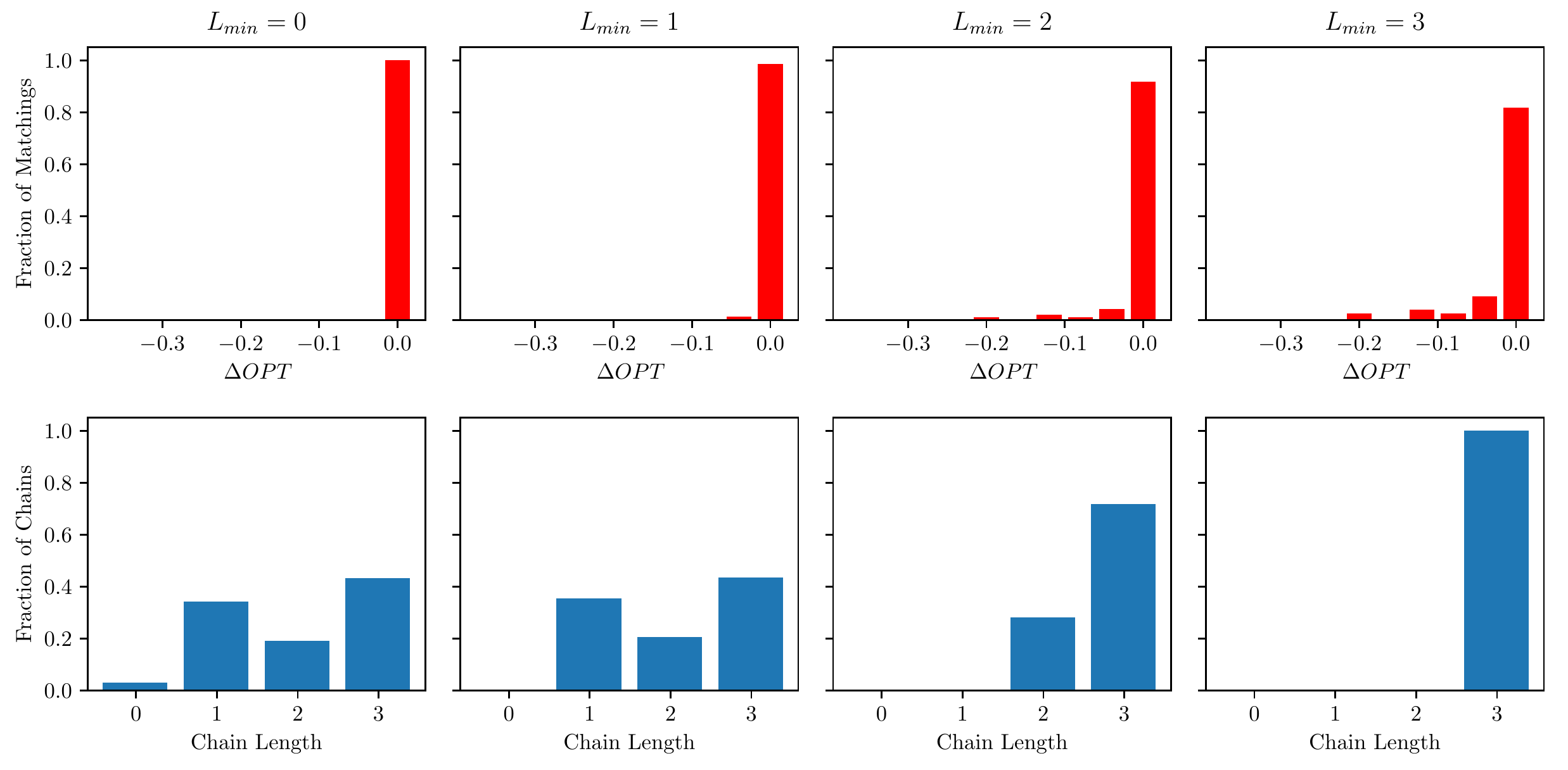}
\caption{$\DeltaOPTraw{}$ (top row) and chain lengths (bottom row) for the optimal matchings with minimum chain length $L_{min}$, and maximum chain length of 3.}\label{fig:chain}
\end{figure*}

As expected, enforcing $L_{min}>0$ results in longer chains -- such that when $L_{min}=L_{max}=3$, \emph{all} chains have length 3. Surprisingly, enforcing a minimum chain length does not impact the overall matching score. Indeed, even when $L_{min}=3$, $70\%$ of all matchings have a \emph{zero optimality gap}. However these experiments did not consider edge failures. As discussed in \Secref{sec:ex}, edge failures impact long cycles and chains more than short cycles and chains; in practice, when edges have a nonzero failure probability, setting a high $L_{min}$ makes the matching more susceptible to failure (i.e. less robust).

\subsection{Edge Existence Robust Formulation}

In this section we develop a mixed integer linear program formulation for the edge existence robust kidney exchange ($KEX(\mU^w_\Gamma)$). This problem maximizes the matching score while minimizing the objective with respect to realized cycle and chain weights $\mathbf{\hat{w}^M_c}$ for the current matching $M$. We develop an edge-existence robust formulation by directly minimizing the PI-TSP Objective (\ref{pitsp:obj}) over all cycle and chain weight realizations in $\mathcal{U}^w_\Gamma$. For brevity, let $\mM^P$ be the set of all possible feasible solutions to the PI-TSP formulation; we represent these feasible solutions as $(\mathbf{y},\mathbf{z})\in \mM$, where $\mathbf{y}$ are the edge decision variables for chains, and $\mathbf{z}$ are the cycle decision variables.

For any feasible solution $(\mathbf{y},\mathbf{z})\in \mM^P$, we find the minimum objective value for any realized cycle and chain weights in $\mathcal{U}^w_\Gamma$ With some abuse of notation, this minimum is represented by the function $Z(\mathbf{y},\mathbf{z})$. In this section we separate the realized weights $\mathbf{\hat{w}_c}$ into the realized cycle weights $\mathbf{\hat{w}^C_c}$ and the realized chain weights $\mathbf{\hat{w}^N_c}$.
\begin{align*}
Z(\mathbf{y},\mathbf{z}) &= \min\limits_{(\mathbf{\hat{w}^C_c},\mathbf{\hat{w}^N_c})\in \mathcal{U}^w_\Gamma} \sum\limits_{n\in N}\hat{w}^N_{n} + \sum\limits_{c\in C} \hat{w}^C_c z_c 
\end{align*}
Note that maximizing $Z(\mathbf{y},\mathbf{z})$ is equivalent to solving $KEX(\mU^w_\Gamma)$ -- the robust kidney exchange problem with uncertainty set $KEX(\mU^w_\Gamma)$. The following lemma states that this is equivalent to solving the constant-budget edge existence robust kidney exchange problem $KEX(\mathcal{U}_\Gamma^{E})$. 
\begin{lemma}
$KEX(\mathcal{U}_\Gamma^{E})$ is equivalent to $KEX(\mU^w_\Gamma)$
\end{lemma}
\begin{proof}
Consider a feasible matching $M=(z_c, y_e)$. The only difference between $KEX(\mathcal{U}_\Gamma^{E})$ and $KEX(\mU^w_\Gamma)$ is the minimization of the objective over uncertainty sets $\mathcal{U}_\Gamma^{E}$ and $\mU^w_\Gamma$ respectively.

Problem $KEX(\mathcal{U}_\Gamma^{E})$ minimizes the matching weight over edge subsets $\hat{E} \subseteq E$, where $R=E \setminus \hat{E}$ contains up to $\Gamma$ edges:
\begin{itemize}
\item If $\Gamma=1$, the largest decrease in matching weight occurs if the \emph{highest weight} cycle or chain is discounted -- that is, if $R$ contains the first edge in the highest weight chain, or any edge in the highest weight cycle.
\item Similarly if $\Gamma=2$, the largest decrease in matching weight occurs when the two highest-weight cycles and chains are discounted.
\end{itemize}

Thus, for any positive $\Gamma$ and any feasible matching $M$, the minimum objective in $KEX(\mathcal{U}_\Gamma^{E})$ occurs when the $\Gamma$ highest-weight cycles and chains in $M$ are discounted.

In $KEX(\mU^w_\Gamma)$, for any $\Gamma$ and any feasible matching $M$, the minimum occurs (trivially) when the $\Gamma$ highest-weight cycles or chains are discounted in $\mU^w_\Gamma$. 

For any matching $M$, minimizing the $KEX$ objective over $\mathcal{U}^{E}$ and $\mU^w_\Gamma$ produce the same outcome -- the $\Gamma$ highest-weight cycles and chains are discounted. Thus, the minimization in $KEX(\mathcal{U}^{E})$ and $KEX(\mU^w_\Gamma)$ is equivalent.
\end{proof}

Thus, to solve the constant-budget edge existence robust kidney exchange problem, we can solve \Probref{eq:robzex} -- which maximizes $Z(\mathbf{y},\mathbf{z})$ over all feasible matchings $(\mathbf{y},\mathbf{z})\in \mM^P$.
\begin{subequations}\label{eq:robzex}
\begin{align} 
\max \quad& Z(\mathbf{y},\mathbf{z}) \\
& (\mathbf{y},\mathbf{z}) \in \mM^P
\end{align}
\end{subequations}

We proceed by solving \Probref{eq:robzex}, which is equivalent to $KEX(\mU^w_\Gamma)$. To solve this problem we first develop a linear formulation for $Z$ using the PC-TSP decision variables, and then we maximize this linear expression. 

\subsection{Linear Formulation for $Z$}

In this section we minimize the function $Z$ for any matching $(\mathbf{y},\mathbf{z})\in \mM^P$, within uncertainty set $\mathcal{U}^w_\Gamma$. Within this uncertainty set, up to $\floor{\Gamma}$ cycles and chains can have zero realized weight (i.e. $\hat{w}_c=0$), and if $\Gamma$ is not integer, then one cycle or chain will have realized weight equal to the fraction $(\Gamma-\floor{\Gamma})$ of its total nominal weight (i.e. $\hat{w}_c=(\Gamma-\floor{\Gamma})w_c$. We say that any cycle or chain $c$ with $\hat{w}_c <  w_c$ is \emph{discounted}. 

First note that if a matching uses $G$ cycles and chains, and $G<\Gamma$, only $G$ objects are discounted. Thus let $\Gamma'=\min\{G,\Gamma\}$ be the number of discounted cycles and chains, i.e.,
$$ G = \sum\limits_{c\in C} z_c + \sum\limits_{n\in N}\sum\limits_{e\in \delta^+(n)} y_e.$$
To linearize the definition of $\Gamma'$, we introduce variable $h$, which is 1 if $G<\Gamma$ and 0 otherwise. The statement $\Gamma'=\min\{G,\Gamma\}$ is linearized using the following constraints:
\begin{align*}
\Gamma - G  &\leq Wh \\
G  - \Gamma &\leq W(1-h) \\
G - Wh  &\leq \Gamma' \\
\Gamma - W(1-h) &\leq \Gamma'  \\
h &\in \{0,1\}
\end{align*}
where $W$ is a large constant.

The function $Z$ is minimized w.r.t. the realized weights, when the  $\Gamma'$ discounted cycles and chains are those with the \emph{largest} weight. To select these objects we introduce variables $g^C_c,g^N_n \in\{0,1 \}$ for each cycle $c\in C$ and each chain's NDD $n\in N$. For any matching, let $m$ be the smallest weight of any discounted cycle or chain -- that is, $m$ is the $\ceil{\Gamma'}^{th}$ highest weight of any cycle or chain used in the matching. We define $g^C_c$ and $g^N_n$ as follows
\begin{equation*}
g^C_c = \begin{cases}
1 &\text{if } w^C_c \geq m \\
0 &\text{otherwise}
\end{cases}
\quad\quad
g^N_n = \begin{cases}
1 &\text{if } w^N_n \geq m \\
0 &\text{otherwise}
\end{cases}
\end{equation*}
Thus $g^C_c=1$ or $g^N_n=1$ implies that cycle $c$ or chain $n$ should be discounted if used in the matching. We define these variables using linear constraints, in two steps. First, we require that $g^{\{C,N\}}_j=1$ only if $g^{\{C,N\}}_k=1$ for all cycles and chains $k$ with weight larger than $w^{\{C,N\}}_j$. That is, we require that variables $g^{\{C,N\}}$ obey the same ordering as $w^{\{C,N\}}$. This ordering requirement can be defined using the following correspondences
\begin{align}
g^C_i \geq g^C_j &\Leftrightarrow w^C_i \geq w^C_j,\hspace{0.2in} i,j\in C  \label{cor:1}\\
g^C_c> g^N_n &\Leftrightarrow w^C_c > w^N_n,\hspace{0.2in} c\in C,n\in N \label{cor:2}\\
g^C_c\leq g^N_n &\Leftrightarrow w^C_c \leq w^N_n,\hspace{0.2in} c\in C,n\in N \label{cor:2a}\\
g^N_i \geq g^N_j &\Leftrightarrow w^N_i \geq w^N_j , \hspace{0.2in} i,j\in N\label{cor:3}
\end{align} 
Note that cycle weights are fixed but chain weights depend on the decision variables. Thus we determine ordering relation \ref{cor:1} by sorting all cycle weights during preprocessing, and enforcing this ordering over $g^C_i$ using the relation $\geq_C$, defined as 
$$ \geq_C = \left\{ (g^C_{a}, g^C_{b}) \in \mathbf{g}^C \times \mathbf{g}^C \mid w^C_{a} \geq w^C_{b}   \right\}. $$
Using this notation, the ordering relation $\geq_C$ contains all pairs of cycles $(a,b)$ such that $w^C_{a} \geq w^C_{b}$. For simplicity, I will denote this ordering relation as
$$ a \geq_C b.$$
This ordering relation is enforced on variables $g^C_{i}$ using  $(|C|-1)$ constraints. The ordering required by correspondence \ref{cor:2}, \ref{cor:2a}, and \ref{cor:3} depend on the chain weights, which in turn depend on decision variables. We can linearize these correspondences using the following inequalities
\begin{align*}
g^C_c + q_{cn} &\geq g^N_n \\
W(1-q_{cn}) &\geq w^C_c - w^N_n \\
\\
g^N_n + (1-q_{cn}) &\geq g^C_c \\
Wq_{cn} &\geq w^N_n - w^C_c \\
\\
& q_{cn}\in \{0,1\},c\in C,n\in N,\\
\end{align*}
Where $W$ is a large constant. When $w^C_c > w^N_n$, this forces $q_{cn}$ to be 0; as a result, the inequality $g^C_c \geq g^N_n $ must hold. Otherwise, if $w^C_c < w^N_n$, this forces $q_{cn}$ to be 1, which forces the inequality $g^N_n  \geq g^C_c $ to hold. 

Similarly, the following constraints enforce the ordering in correspondence \ref{cor:3} over variables $g^N_n$
\begin{align*}
g^N_i + q^N_{ij} &\geq g^N_j \\
W(1-q^N_{ij}) &\geq w^N_i - w^N_j \\
\\
g^N_j + (1-q^N_{ij}) &\geq g^N_i \\
Wq^N_{ij} &\geq w^N_j - w^N_i \\
\\
& q^N_{ij}\in \{0,1\},i,j\in N,i\neq j\\
\end{align*}

Next we require that only $\Gamma'$ objects are discounted. Note that if cycle $c$ is discounted if $g^C_c w^C_c=1$, and chain $n$ is discounted if $g^{N}_n\sum\limits_{e\in E}y^n_e =1$. Thus, the following identity requires that exactly $\Gamma
$ objects are discounted:
$$\sum\limits_{c\in C}z_c g^C_c +   \sum\limits_{n \in N} g^N_n \sum\limits_{e\in \delta^+(n)}y_e= \Gamma'$$

We use these variables to directly minimize $Z( \mathbf{y},\mathbf{z})$ w.r.t. $\mathcal{U}^w_\Gamma$, and the result is given in \Eqref{eq:robsolex}.
\begin{subequations}\label{eq:robsolex}
\begin{align}
Z( \mathbf{y},\mathbf{z})=\quad& \sum\limits_{n\in N} w^N_{n} + \sum\limits_{c\in C}w_c z_c - \sum\limits_{n\in N}g^N_n w^N_{n} - &&\sum\limits_{c\in C}g^C_c w_c z_c \\
\text{s.t.} \quad&   \Gamma - G  \leq Wh \\
\quad& G  - \Gamma \leq W(1-h) \\
\quad&  G - Wh  \leq \Gamma' \\
\quad&  \Gamma - W(1-h) \leq \Gamma'  \\
\quad&  \begin{array}{rl} g^C_c + q_{cn} &\geq g^N_n \\
 W(1-q_{cn}) &\geq w^C_c - w^N_n \\
 g^N_n + (1-q_{cn}) &\geq g^C_c \\
 Wq_{cn} &\geq w^N_n - w^C_c  \end{array} &&, c\in C, n\in N \\
\quad &\begin{array}{rl}  g^N_i + q^N_{ij} &\geq g^N_j \\
W(1-q^N_{ij}) &\geq w^N_i - w^N_j  \end{array} &&, q^N_{ij}\in \{0,1\},i,j\in N,i\neq j\\
\quad &\sum\limits_{c\in C}z_c g^C_c +   \sum\limits_{n \in N} g^N_n \sum\limits_{e\in \delta^+(n)}y_e= \Gamma' \\
\quad & g^C_i \geq_C g^C_j &&, i,j\in C, i\neq j \\
\quad &q_{cn} \in \{0,1\}, &&c\in C,n\in N,\\
\quad &h \in \{0,1\}
\end{align}
\end{subequations}

Note that there are two sets of quadratic expressions in this formulation: $g^C_c z_c$, and $w^N_n g^N_n$. These are linearized in the next section, which addresses non-integer $\Gamma$.


\subsection{Non-Integer $\Gamma$}
When $\Gamma$ is not integer, the actual number of discounted cycles and chains ($\Gamma'$) may be integer or non-integer valued.  When $\Gamma'$ is not integer valued, up to $\floor{\Gamma'}$ cycles and chains are fully discounted (i.e. $\hat{w}_c=0$), and the smallest-weight cycle or chain is discounted by fraction $(\Gamma-\floor{\Gamma})$. We include this fractional discount by using two sets of indicator variables $f^{\{C,N\}}_i$ and $p^{\{C,N\}}_i$ for all cycles and chains $i\in C\cup N$, and then discount each $i$ as follows:
\begin{itemize}
\item $i$ is fully discounted if $p^{\{C,N\}}_i=f^{\{C,N\}}_i=1$.
\item $i$ is partially discounted fraction $(\Gamma-\floor{\Gamma})$ if $f^{\{C,N\}}_i=0$ and $p^{\{C,N\}}_i=1$
\item $i$ is not discounted if $f^{\{C,N\}}_i=p^{\{C,N\}}_i=0$.
\end{itemize}

Thus if $\Gamma'$ is integer, $f^{\{C,N\}}_i=p^{\{C,N\}}_i$ for all cycles and chains $i$; if $\Gamma'$ is not integer, then $\ceil{\Gamma'}$ cycles and chains are least partially discounted ($p^{\{C,N\}}_e=1$), and $\floor{\Gamma'}$ cycles and chains are fully discounted ($p^{\{C,N\}}_i=g^f_i=1$). These indicator variables are defined in the same way as $g^{\{C,N\}}_i$ in \Eqref{eq:robsolex}: $p^{\{C,N\}}_i,f^{\{C,N\}}_i\in \{0,1\}$, and they obey the same ordering relation as the cycle and chain weights. However, the number of cycles and chains with $f^{\{C,N\}}_i=1$ can be different than the number of cycles and chains with with $p^{\{C,N\}}_i=1$. Thus we add new constraints for each of these variables. 

\paragraph{Setting the number of discounted cycles and chains.} First we require $\ceil{\Gamma'}$ cycles and chains have $p^{\{C,N\}}_i=1$. Recall that $G$ is the number of matching edges, and $\Gamma'=\min(\Gamma,G)$; if $\Gamma<G$, then $\ceil{\Gamma'}=\ceil{\Gamma}$, and $\ceil{\Gamma'}=G$ otherwise. The variable $h$ is defined to be 1 if $G < \Gamma$ and $0$ otherwise. Thus, the following constraint requires that $\ceil{\Gamma'}$ cycles and chains have $p^{\{C,N\}}_i=1$:
$$ \sum\limits_{n\in N} p^N_n  \sum\limits_{e\in \delta^+(n)}y_e + \sum\limits_{c\in C}p^C_c z_c = h G + (1-h)\ceil{\Gamma}.$$
Similarly, the following constraint requires that $\floor{\Gamma'}$ cycles and chains have $f^{\{C,N\}}_i=1$:
$$ \sum\limits_{n\in N} f^N_n  \sum\limits_{e\in \delta^+(n)}y_e + \sum\limits_{c\in C}f^C_c z_c= h G + (1-h)\floor{\Gamma}.$$
Thus if $G<\Gamma$, then all $G$ cycles and chains will have $f^{\{C,N\}}_i=p^{\{C,N\}}_i=1$; otherwise, there are $\ceil{\Gamma}$ cycles and chains with $p^{\{C,N\}}_i=1$, and $\floor{\Gamma}$ cycles and chains with $f^{\{C,N\}}_i=1$, where the partially-discounted cycle or chain has $f^{\{C,N\}}_i=0$ and $p^{\{C,N\}}_i=1$.

\paragraph{Ordering relation over indicator variables.} To enforce the ordering relation over indicator variables $f^N_n$, $p^N_n$, $f^C_c$, and $p^C_c$, we use constraints similar to those used in the edge weight robust formulation. The auxiliary variables $q_{cn}$ and $q^N_{ij}$ are defined the same way here: $q_{cn}$ is $0$ when $w^C_n > w^N_n$ and $1$ otherwise; $q^N_{ij}$ is $0$ if 

\begin{align*}
\begin{array}{rl}
f^C_c + q_{cn} &\geq f^N_n \\
p^C_c + q_{cn} &\geq p^N_n \\
f^N_n + (1-q_{cn}) &\geq f^C_c \\
p^N_n + (1-q_{cn}) &\geq p^C_c \\
W(1-q_{cn}) &\geq w^C_c - w^N_n \\
Wq_{cn} &\geq w^N_n - w^C_c 
\end{array} &,c\in C,n\in N,\\
\\
 q_{cn}\in \{0,1\}&,c\in C,n\in N,\\
\end{align*}
Where $W$ is a large constant. When $w^C_c > w^N_n$, this forces $q_{cn}$ to be 0; as a result, the inequality $f^C_c \geq f^N_n $ and $p^C_c \geq p^N_n $ must hold. Otherwise, if $w^C_c < w^N_n$, this forces $q_{cn}$ to be 1, which forces the inequality $f^N_n  \geq f^C_c $ and $p^N_n  \geq p^C_c $ to hold. 

Similarly, the following constraints enforce the ordering in correspondence \ref{cor:3} over variables $f^N_n$ and $p^N_n$
\begin{align*}
\begin{array}{rl}
f^N_i + q^N_{ij} &\geq f^N_j \\
p^N_i + q^N_{ij} &\geq p^N_j \\
f^N_j + (1-q^N_{ij}) &\geq f^N_i \\
p^N_j + (1-q^N_{ij}) &\geq p^N_i \\
Wq^N_{ij} &\geq w^N_j - w^N_i \\
W(1-q^N_{ij}) &\geq w^N_i - w^N_j
\end{array} &, i,j\in N,i\neq j\\
\\
q^N_{ij}\in \{0,1\}&,i,j\in N,i\neq j\\
\end{align*}

As before, correspondence \ref{cor:1}, the ordering between cycle indicator variables, is enforced using the pre-determined ordering $\geq_C$.
\begin{align*}
\begin{array}{rl}
f^C_a &\geq_C f^C_b  \\
p^C_a &\geq_C p^C_b
\end{array}
&, a,b \in C, a\neq b 
\end{align*}

\paragraph{Objective for non-integer $\Gamma$.} 

Using these indicator variables, the new objective of the robust formulation is
\begin{align*} \max  \sum\limits_{n\in N} w^N_{n} + \sum\limits_{c\in C}w_c z_c &- (1-\Gamma + \floor{\Gamma})\left(\sum\limits_{n\in N} w^N_{n} f^N_n + \sum\limits_{c\in C}f^C_c w_c z_c \right)\\
& -  (\Gamma - \floor{\Gamma})\left(\sum\limits_{n\in N} w^N_{n} p^N_n + \sum\limits_{c\in C}p^C_c w_c z_c  \right)
\end{align*}
which discounts cycle or chain $i$ by its full weight if $f^{\{C,N\}}_i=p^{\{C,N\}}_i=1$, and by fraction $\left(\Gamma - \floor{\Gamma} \right)$ of its weight if $f^{\{C,N\}}_i=0$ and $p^{\{C,N\}}_i=1$.

\paragraph{Non-linear terms.} There are now $7$ types of nonlinear terms in this formulation: 
\begin{itemize}
\item $hG$,
\item $w^N_n f^N_n$, 
\item $w^N_n p^N_n$, 
\item $z_c f^C_c$, 
\item $z_c p^C_c$, 
\item $f^N_n y_e$, and 
\item $p^N_n y_e$. 
\end{itemize}
First we linearize the chain-related quadratic terms by introducing the variables $\hat{f}^N_n \equiv w^N_n f^N_n$ and $\hat{p}^N_n \equiv w^N_n p^N_n$. The following constraints define these new variables, using a large constant $W$. 

\begin{align*}
\begin{array}{rl}
\hat{f}^N_n  &\leq  f^N_n W\\
\hat{f}^N_n  &\leq w^N_n \\
\hat{f}^N_n  &\geq w^N_n - (1-f^N_n)W
\end{array} 
&, n\in N\\
\hat{f}^N_n  \geq 0 &, n\in N \\
\\
\begin{array}{rl}
\hat{p}^N_n  &\leq  p^N_n W\\
\hat{p}^N_n  &\leq w^N_n \\
\hat{p}^N_n  &\geq w^N_n - (1-p^N_n)W
\end{array} 
&, n\in N\\
\hat{p}^N_n  \geq 0 &, n\in N \\
\end{align*}

Next we define variables $\hat{f}^C_c \equiv z_c f^C_c$ and $\hat{p}^C_c \equiv z_c p^C_c$ using the following constraints.

\begin{align*}
\begin{array}{rl}
\hat{f}^C_c &\leq f^C_c \\
\hat{f}^C_c &\leq z_c \\
\hat{f}^C_c&\geq f^C_c+z_c - 1 
\end{array} 
&, c\in C\\
\hat{f}^C_c \in \{0,1\} &, c\in C \\
\\
\begin{array}{rl}
\hat{p}^C_c &\leq p^C_c \\
\hat{p}^C_c &\leq z_c \\
\hat{p}^C_c&\geq p^C_c+z_c - 1 
\end{array} 
&, c\in C\\
\hat{p}^C_c \in \{0,1\} &, c\in C \\
\end{align*}

To linearize the term $hG$, we introduce variable $\hat{g}\equiv hG$, which is defined using the following constraints. As before, $W$ is a large constant.
\begin{align*}
\hat{h} &\leq hW \\
\hat{h} &\leq G \\
\hat{h} &\geq G - (1-h)W \\
\hat{h} &\geq 0
\end{align*}

Finally, we introduce the variables $ F_n \equiv f^N_n  \sum\limits_{e\in \delta^+(n)}y_e$ and $ P_n \equiv p^N_n  \sum\limits_{e\in \delta^+(n)}y_e$, defined with the following constraints. Note that for each NDD $n\in N$ the sum of all $y_e$ variables is either zero (if $n$ does not initiate a chain) or $1$ (if $n$ initiates a chain). Thus $F_n$ and $P_n$ are products of binary variables, which we define using the following constraints.

\begin{align*}
\begin{array}{rl}
F_n &\leq f^N_n \\
F_n &\leq \sum\limits_{e\in \delta^+(n)}y_e\\
F_n &\geq \sum\limits_{e\in \delta^+(n)}y_e +  f^N_n - 1 
\end{array} 
&, c\in C\\
F_n \in \{0,1\} &, n\in N \\
\\
\begin{array}{rl}
P_n &\leq p^N_n \\
P_n &\leq \sum\limits_{e\in \delta^+(n)}y_e\\
P_n &\geq \sum\limits_{e\in \delta^+(n)}y_e +  p^N_n - 1 
\end{array} 
&, n\in N\\
P_n \in \{0,1\} &, n\in N \\
\end{align*}

\paragraph{Linear formulation.} Finally, for any feasible matching we directly minimize $Z$ by discounting the $\Gamma'$ largest-weight cycles and chains. This is accomplished using the variables  $\hat{f}^N_n$, $\hat{p}^N_n$, $\hat{f}^C_c$, $\hat{p}^C_c$. \Eqref{eq:linearzex} gives the minimization of $Z$ for any matching $( \mathbf{y},\mathbf{z})$ , using only linear constraints. 

\begin{subequations}\label{eq:linearzex}
\begin{align}
\begin{array}{r}
Z( \mathbf{y},\mathbf{z})=  \quad \sum\limits_{n\in N} w^N_{n} + \sum\limits_{c\in C}w_c z_c - (1-\Gamma + \floor{\Gamma})\left(\sum\limits_{n\in N} \hat{f}^N_n + \sum\limits_{c\in C}\hat{f}^C_c w_c \right)\\
 -  (\Gamma - \floor{\Gamma})\left(\sum\limits_{n\in N} \hat{p}^N_n + \sum\limits_{c\in C}\hat{p}^C_c w_c \right) \\
 \end{array} \label{linearzex:obj}\\
\intertext{s.t.} \nonumber\\
\begin{array}{rl} \Gamma - G  &\leq Wh  \\
\quad G  - \Gamma &\leq W(1-h) \\
\quad G - Wh  &\leq \Gamma' \\
\quad \Gamma - W(1-h) &\leq \Gamma'  
\end{array} \\
\quad \sum\limits_{n\in N} P_n + \sum\limits_{c\in C}\hat{p}^C_c = \hat{h} + (1-h)\ceil{\Gamma} \\
\quad \sum\limits_{n\in N} F_n + \sum\limits_{c\in C}\hat{f}^C_c = \hat{h} + (1-h)\floor{\Gamma} \\
\quad \begin{array}{rl}
f^C_c + q_{cn} &\geq f^N_n \\
p^C_c + q_{cn} &\geq p^N_n \\
f^N_n + (1-q_{cn}) &\geq f^C_c \\
p^N_n + (1-q_{cn}) &\geq p^C_c \\
W(1-q_{cn}) &\geq w^C_c - w^N_n \\
Wq_{cn} &\geq w^N_n - w^C_c 
\end{array} &,\hspace{0.1in} c\in C,n\in N\\
\quad \begin{array}{rl}
f^N_i + q^N_{ij} &\geq f^N_j \\
p^N_i + q^N_{ij} &\geq p^N_j \\
f^N_j + (1-q^N_{ij}) &\geq f^N_i \\
p^N_j + (1-q^N_{ij}) &\geq p^N_i \\
Wq^N_{ij} &\geq w^N_j - w^N_i \\
W(1-q^N_{ij}) &\geq w^N_i - w^N_j
\end{array} &,\hspace{0.1in} i,j\in N,i\neq j\\
\quad \begin{array}{rl}
f^C_a &\geq_C f^C_b  \\
p^C_a &\geq_C p^C_b
\end{array}
&,\hspace{0.1in} a,b \in C, a\neq b \\ 
\quad \begin{array}{rl}
\hat{f}^N_n  &\leq  f^N_n W\\
\hat{f}^N_n  &\leq w^N_n \\
\hat{f}^N_n  &\geq w^N_n - (1-f^N_n)W
\end{array} 
&,\hspace{0.1in} n\in N\\
\quad \begin{array}{rl}
\hat{p}^N_n  &\leq  p^N_n W\\
\hat{p}^N_n  &\leq w^N_n \\
\hat{p}^N_n  &\geq w^N_n - (1-p^N_n)W
\end{array} 
&,\hspace{0.1in} n\in N\\
\begin{array}{rl}
\hat{f}^C_c &\leq f^C_c \\
\hat{f}^C_c &\leq z_c \\
\hat{f}^C_c&\geq f^C_c+z_c - 1 
\end{array} 
&,\hspace{0.1in} c\in C\\
\begin{array}{rl}
\hat{p}^C_c &\leq p^C_c \\
\hat{p}^C_c &\leq z_c \\
\hat{p}^C_c&\geq p^C_c+z_c - 1 
\end{array} 
&,\hspace{0.1in} c\in C\\
\begin{array}{rl}
\hat{h} &\leq hW \\
\hat{h} &\leq G \\
\hat{h} &\geq G - (1-h)W \\
\hat{h} &\geq 0 
\end{array} \\
\begin{array}{rl}
F_n &\leq f^N_n \\
F_n &\leq \sum\limits_{e\in \delta^+(n)}y_e\\
F_n &\geq \sum\limits_{e\in \delta^+(n)}y_e +  f^N_n - 1 
\end{array} &, c\in C\\
\begin{array}{rl}
P_n &\leq p^N_n \\
P_n &\leq \sum\limits_{e\in \delta^+(n)}y_e\\
P_n &\geq \sum\limits_{e\in \delta^+(n)}y_e +  p^N_n - 1 
\end{array} &, n\in N\\
\hat{f}^C_c ,\hat{p}^C_c \in \{0,1\} &,\hspace{0.1in} c\in C \\
\hat{f}^N_n, \hat{p}^N_n  \geq 0 &,\hspace{0.1in} n\in N \\
F_n \in \{0,1\} &, n\in N \\
P_n \in \{0,1\} &, n\in N \\
q^N_{ij}\in \{0,1\} &,\hspace{0.1in} i,j\in N,i\neq j\\
 q_{cn}\in \{0,1\}&,\hspace{0.1in}c\in C,n\in N\\
\quad h \in \{0,1\}
\end{align}
\end{subequations}

The linear formulation for $KEX(\mU^w_\Gamma)$ is obtained by adding the PI-TSP constraints to \Probref{eq:linearzex}, and mazimizing the objective \ref{linearzex:obj}.

This linear formulation can be solved by any standard solver; our experiments use Gurobi~\cite{Gurobi}.

%% file: APP_fairness.tex
In this section we use the framework of edge weight uncertainty to address the problem of fairness in kidney exchange. Though seemingly unrelated, fairness and uncertainty share some key characteristics. The concept of \emph{budgeted uncertainty} balances the nominal objective value with the worst case. A similar trade-off exists between fairness and efficiency in kidney exchange: allocating kidneys to hard-to-match patients is \emph{fair}, but often reduces the number of possible transplants.

\subsection{The Price of Fairness} 

In kidney exchange, fairness often pertains to \emph{highly-sensitized} patients, who are very unlikely to find a compatible donor. Highly-sensitized patients face longer waiting times than lowly-sensitized patients\footnote{{\texttt{https://optn.transplant.hrsa.gov/data/}}}. In part this is because highly sensitized patients are hard to match; for this reason most kidney exchange optimization algorithms -- which maximize matching size or weight -- marginalize highly-sensitized patients.

A patient's sensitization level is measured by her Calculated Panel Reactive Antibody (CPRA) score, which ranges from $0$ to $100$. Patient-donor pair vertices in the exchange graph are highly-sensitized if the pair's patient has a CPRA score above some threshold $\tau$, which is set by policymakers ($\tau=80$ is common). Let $V_H$ ($V_L$) be the set of highly-sensitized (lowly-sensitized) vertices in $P$, and let $E_H$ ($E_L$) be the set of all edges that end in $V_H$ ($V_L$).

Fairness for a matching $M$ is often quantified using the \emph{utility} assigned to $V_H$ and $V_L$ -- i.e. the sum of edge weights into each vertex set,
$$ U_H(M) = \sum\limits_{e\in E_H} x_e w_e, \hspace{0.2in} U_L(M) = \sum\limits_{e\in E_L} x_e w_e.$$

The \emph{utilitarian} utility function is defined as $u(M)=U_H(M)+U_L(M)$ (i.e. the total edge weight of matching $M$). We might define a \emph{fair} utility function $u_f : \mM \to \mathbb{R}$, such that the matching $M^*_f$ that maximizes $u_f$ is considered fair:
\begin{align*}
  M^*_f &= \arg\max_{M \in \mM} u_f(M) 
\end{align*}

Fairness is quantified using the \emph{fraction of the fair score} $\%F:M,\mM\rightarrow [0,1]$ -- i.e. the fraction of the maximum possible utility awarded to highly sensitized patients
$$ \%F(M,\mM) = U_H(M) / \max\limits_{M'\in \mM} U_H(M').$$

\citet{Bertsimas11:Price} defines the \emph{price of fairness} as the ``relative system efficiency loss under a fair allocation assuming that a fully efficient allocation is one that maximizes the sum of [participant] utilities.''  Thus the price of fairness is defined using the set of matchings $\mM$, the fair utility function $u_f$, and the utilitarian utility function $u$:
\begin{align} \label{eq:pof}
  \POF(\mM,u_f)= \frac{u\left(M^*_{\text{ }}\right) - u\left(M^*_f\right)}{u\left(M^*\right)}
\end{align}

$\POF(\mM,u_f)$ is the relative loss in (utilitarian) efficiency caused by choosing the fair outcome $M^*_f$ rather than the most efficient outcome. 

Balancing $\%F$ and $\POF$ is a key problem in kidney exchange. Achieving a high degree of fairness (high $\%F$) often incurs a high $\POF$; on the other hand, requiring a low $\POF$ ofen results in low $\%F$. \citet{Dickerson14:Price} propose two rules for enforcing fairness in kidney exchange, and demonstrate that without chains, the price of fairness is low in theory. \citet{McElfresh18:fair} extended this result, finding that adding chains lowers the theoretical price of fairness -- eventually to zero; they also propose a fairness rule that limits the price of fairness. 

In the next section we generalize one of the fairness rule proposed by \citet{Dickerson14:Price} using the framework of budgeted robust optimization, and demonstrate its versatility in balancing fairness and efficiency.

\subsection{Fairness Through Robustness}

In this section we adapt the concept of budgeted uncertainty to apply budgeted \emph{prioritization} to highly sensitized patients in kidney exchange. To prioritize certain patients over others, we assign each edge $e\in E$ a \emph{priority weight} $\hat{w}_e\in [0,\infty)$, equal to the nominal weight multiplied by a factor $(1+\alpha_e)$, with $\alpha_e \in [-1,\infty)$. There are many ways to prioritize highly sensitized vertices using priority weights: we may set $\alpha>0$ for all edges in $E_H$, or we may set $\alpha=-1$ for edges in $E_L$, and so on. 

To balance fairness with efficiency it reasonable to \emph{limit} the degree of prioritization. To limit prioritization, we define a \emph{budgeted prioritization set} $\mathcal{P}$, which bounds the sum of absolute differences between each $w_e$ and $\hat{w}_e$; this prioritization set is given in \Eqref{eq:pset}.

\begin{equation}\label{eq:pset}
\mathcal{P}_\Gamma = \left\{ \mathbf{\hat{w}} \mid \hat{w}_e = w_e (1+\alpha_e),\alpha_e \in [-1,\infty], \sum\limits_{e\in E}\alpha_e w_e \leq \Gamma \right\}
\end{equation}

To prioritize $V_H$, we define $\alpha_e$ differently for each edge $e$. In one type of approach, we prioritize $V_H$ by setting $\alpha_e$ to a constant ($\alpha$) for all $e\in E_H$. This approach is given by $\mathcal{P}^+_\Gamma$, in \Eqref{eq:psetplus}

\begin{equation}\label{eq:psetplus}
\mathcal{P}^+_\Gamma = \left\{ \mathbf{\hat{w}} \mid \hat{w}_e = \begin{cases} w_e (1+\alpha) &\text{if } e\in E_H\\ w_e &\text{otherwise} \end{cases},\alpha \geq 0, \alpha \sum\limits_{e\in E} w_e \leq \Gamma \right\}
\end{equation}

A different type of approach prioritizes $V_H$ by reducing all edges into $E_L$; this approach is given by $\mathcal{P}^-_\Gamma$, in \Eqref{eq:psetminus}.

\begin{equation}\label{eq:psetminus}
\mathcal{P}^-_\Gamma = \left\{ \mathbf{\hat{w}} \mid \hat{w}_e = \begin{cases} w_e (1-\alpha) &\text{if } e\in E_L \\ w_e &\text{otherwise} \end{cases},\alpha \in[0,1], \alpha \sum\limits_{e\in E} w_e \leq \Gamma \right\}
\end{equation}

To apply prioritization to kidney exchange, we either minimize or maximize the kidney exchange objective with respect to $\mathcal{P}$. By choosing $\alpha_e$ and prioritization budget $\Gamma$, this general framework can implement a wide variety of prioritization requirements. Next we show how budgeted prioritization generalizes a previous fairness rule.

\subsection{Weighted Fairness}

Weighted fairness was proposed by \citet{Dickerson14:Price} to prioritize highly sensitized patients in kidney exchange. This fairness rule maximizes the total matching weight, after multiplying all edge weights into highly sensitized patients by a factor $(1+\gamma)$, where parameter $\gamma$ is set by policymakers. Weighted fairness is equivalent to \emph{maximizing} the kidney exchange objective over the budgeted prioritization set $\mathcal{P}^{w}$, given below. This prioritization set is equivalent to $\mathcal{P}^+_\Gamma$, with prioritization budget $\Gamma$ equal to $\gamma$ times the total weight received by highly sensitized patients. 
\begin{equation}\label{eq:psetweighted1}
\mathcal{P}^w_\gamma = \left\{ \mathbf{\hat{w}} \mid \hat{w}_e = \begin{cases} w_e (1+\alpha) &\text{if } e\in E_H \\ w_e &\text{otherwise} \end{cases},\alpha \geq 0, \alpha \sum\limits_{e\in E_H} w_e \leq \gamma \sum\limits_{e\in E_H} w_e  \right\}
\end{equation}

Note that the uncertainty budget does not depend on edge weights, and can be written succinctly as \Eqref{eq:psetweighted}.

\begin{equation}\label{eq:psetweighted}
\mathcal{P}^w_\gamma = \left\{ \mathbf{\hat{w}} \mid \hat{w}_e = \begin{cases} w_e (1+\alpha) &\text{if } e\in E_H \\ w_e &\text{otherwise} \end{cases}, 0 \leq \alpha \leq \gamma   \right\}
\end{equation}

Weighted fairness is implemented by \emph{maximizing} over priority set $\mathcal{P}^w_\gamma$, as in \Probref{eq:wtfairness}

\begin{subequations} \label{eq:wtfairness}
\begin{align}
\max \max\limits_{\hat{w}\in \mathcal{P}^w_\gamma} &\hat{w} \cdot x_e \\
\mathbf{x} &\in \mM
\end{align}
\end{subequations}

\begin{proposition}
$\gamma$-weighted fairness is equivalent to maximizing the kidney exchange objective over $\mathcal{P}^w_\gamma$.
\end{proposition}

As demonstrated in \Eqref{eq:psetweighted1}, weighted fairness uses the prioritization budget $\Gamma =\gamma\sum\limits_{e\in E_H}w_e$, which is proportional to the weight received by highly sensitized patients. Thus, we may derive an upper bound on the $\POF$ for $\gamma$-weighted fairness.

\begin{proposition}\label{prop:wtpofgen}
For $\gamma$-weighted fairness, and some matching $M$ the price of fairness for choosing matching $M$ is bounded above by
$$ \POF(u^w_\gamma,M) \leq \frac{\gamma}{1+\gamma + U_L(M)/U_H(M)}.$$
\end{proposition}

\ifshowproofs
\begin{proof}
Suppose that $\gamma$-weighted fairness chooses matching $M$ over a higher-weight matching $E$. In the worst case, both $F$ and $E$ receive nearly the same \emph{priority weight} under $\gamma$-weighted fairness (within a small perturbation $\epsilon$). Let the utility awarded by each outcome to highly- and lowly-sensitized patients be given by
\begin{equation*}
\begin{array}{rlrl}
U_H(M) & = A & U_L(M) &= B \\
U_H(E) & = 0 & U_L(E) &= A(1+\gamma) + B - \epsilon \\
\end{array}
\end{equation*}
with $0<\epsilon \ll 1$. Both $M$ and $E$ receive nearly the same priority weight from $\gamma$-weighted fairness, but $E$ receives $\gamma A$ more weight than $M$:
\begin{align*}
u^w_\gamma(M) &= A(1+\gamma) + B \\
u^w_\gamma(E) &= A(1+\gamma) + B - \epsilon
\end{align*}
And thus $\gamma$-weighted fairness selects $M$ over $E$. Taking the limit as $\epsilon \rightarrow 0$, the price of fairness for choosing $M$ is
$$ \POF(u^w_\gamma,M) = \frac{A(1+\gamma) + B - \epsilon - A - B}{A(1+\gamma) + B - \epsilon} = \frac{\gamma}{1+\gamma + B/A},$$
note that $A=U_H(M)$ and $B=U_L(M)$, and thus
$$ \POF(u^w_\gamma,M) = \frac{\gamma}{1+\gamma + U_L(M)/U_H(M)}.$$
Note that this is the worst-case $\POF$ for choosing $M$, and thus
$$ \POF(u^w_\gamma,M) \leq \frac{\gamma}{1+\gamma + U_L(M)/U_H(M)}.$$
\end{proof}
\fi

It follows that this $\POF$ is maximized when $U_L(M)=0$, which is the worst case $\POF$ for $\gamma$-weighted fairness. 

\begin{corollary}\label{cor:wtpof}
For $\gamma$-weighted fairness, the price of fairness  is bounded above by
$$ \POF(u^w_\gamma) \leq \frac{\gamma}{1+\gamma} $$
\end{corollary}


\begin{proposition} \label{prop:wtfairnessgen}
Let $U_H^*$ be the maximum possible utility for highly-sensitized patients. For $\gamma$-weighted fairness, and some matching $M$ the fraction of the fair score $\%F$ for matching $M$ is bounded below by
$$ \%F(M,\mM) \geq 1 - \frac{U_L(M)}{U_H^*}\frac{1}{1+\gamma}.$$
\end{proposition}

\ifshowproofs
\begin{proof}
Let $M\in mM$ be a feasible matching, and let $U_H^*$ be the maximum possible utility for highly-sensitized patients. Consider the worst case scenario for $\gamma$-weighted fairness: two outcomes receive nearly equal utility from $\gamma$-weighted fairness, but the outcome chosen is far less fair. Let the \emph{fair} outcome $F$ assign the maximum possible utility to highly sensitized patients, and zero utility to lowly sensitized patients:
$$ u_H(F) = U_H^*, \hspace{0.2in} u_L(F) = 0.$$
Let $M$ be the outcome selected by $\gamma$-weighted fairness, which assigns utility $\beta U_H(M)$ to highly sensitized patients, with $0<\beta <1$, and some utility $A+\epsilon$ to lowly sensitized patients, with $0<\epsilon\ll 1$:
$$ u_H(M) = \beta U_H^*, \hspace{0.2in} u_L(M) = A + \epsilon,$$
and note that $\beta$ is $\%F$, the fraction of the fair score, for outcome $M$. 

Letting $\epsilon\rightarrow 0$, both $F$ and $M$ receive the same utility under $\gamma$-weighted fairness; that is,
$$ U_H^* (1+\gamma) = \beta U_H^* (1+\gamma) + U_L(M).$$
Rearranging, we have
$$\%F(M,\mM) = \beta = 1 - \frac{U_L(M)}{U_H^*}\frac{1}{1+\gamma}.$$
This is the worst-case outcome for $\%F$, and thus
$$\%F(M,\mM) \geq \beta = 1 - \frac{U_L(M)}{U_H^*}\frac{1}{1+\gamma}.$$
\end{proof}
\fi

It follows that the worst-case $\%F$ occurs when $U_L$ is maximal, and $M=M^*$.

\begin{corollary} \label{cor:wtfairness}
Let $U_H^*$ and $U_L^*$ be the maximum possible utility for highly- and lowly-sensitized patients, respectively. Under $\gamma$-weighted fairness, the fraction of the fair score $\%F$ is bounded below by
$$ \%F(*,\mM) \geq 1 - \frac{U_L^*}{U_H^*}\frac{1}{1+\gamma}.$$
\end{corollary}


These results may be used to balance $\%F$ and $\POF$, subject to policymaker requirements. For example, suppose policymakers require that $\%F\geq f$, and $\POF \leq p$, for some constants $f$ and $p$. If we know the maximum utility for highly- and lowly- sensitized patients, we can bound $\gamma$ using the worst-case bounds from Corollary \ref{cor:wtpof} and \ref{cor:wtfairness}. Inverting the bounds from these Corollaries with $p=\POF$ and $f=\%f$, we have
$$ \gamma \leq \frac{p}{1-p}, \hspace{0.2in} \gamma \geq \frac{U_L^*}{U_H^*}\frac{1}{1-f} - 1.$$
Combining these restrictions, we arrive at the bounded prioritization set $\mathcal{P}^f_p$, given in \Eqref{eq:wtpsetfp}.
\begin{equation}\label{eq:wtpsetfp}
\mathcal{P}^f_p = \left\{ \mathbf{\hat{w}} \mid \hat{w}_e = \begin{cases} w_e (1+\gamma) &\text{if } e\in E_H \\ w_e &\text{otherwise} \end{cases}, \frac{U_L^*}{U_H^*}\frac{1}{1-f} - 1 \leq \gamma \leq \frac{p}{1-p}   \right\}
\end{equation}
There are two important observations about this prioritization set. First, not all choices of $f$ and $p$ are valid, and this depends on $U_L^*/U_H^*$; that is, choosing either $f$ or $p$ necessarily bounds the other. Second, there are many ways to use $\mathcal{P}^f_p$ in practice: we might \emph{minimize} or \emph{maximize} $\mathbf{\hat{w}}$ before maximizing the kidney exchange objective (i.e., setting $\gamma$ to its maximum or minimum value; this is equivalent to the weighted fairness proposed by \citet{Dickerson14:Price}.

Alternatively, we might allow $\gamma$ to vary within the range of set by $\mathcal{P}^f_p$. This approach allows the optimization algorithm to \emph{choose} the value of $\gamma$, such that priority weight is maximized. Note that this is not equivalent to weighted fairness (\Probref{eq:wtfairness}), which maximizes priority weight \emph{before} maximizing the objective. This variable-$\gamma$ approach is given in \Probref{eq:wtfairvarg}.
\begin{subequations} \label{eq:wtfairvarg}
\begin{align*}
\max\limits_{\hat{w}\in \mathcal{P}^f_p} &\sum\limits_{e\in E} \hat{w}_e \cdot x_e \\
\mathbf{x} &\in \mM
\end{align*}
\end{subequations}
By directly applying the definition of $\hat{w}$ to this problem, we arrive at \Probref{eq:wtfairvarg2}.
\begin{subequations} \label{eq:wtfairvarg2}
\begin{align}
\max \quad &  (1+\gamma) \sum\limits_{e\in E_H} w_e \cdot x_e +\sum\limits_{e\in E_L} w_e \cdot x_e\\
&\frac{U_L^*}{U_H^*}\frac{1}{1-f} - 1 \leq \gamma\leq \frac{p}{1-p} \\
&\mathbf{x} \in \mM
\end{align}
\end{subequations}

In the next section we tighten this the bound on $\%F$ for $\gamma$-weighted fairness, by relaxing the bounds on $\gamma$.

\subsubsection{Variable Weighted Fairness}

The bounds in Corollary \ref{cor:wtpof} and \ref{cor:wtfairness} are for the \emph{worst-case} bounds on $\gamma$; however, the worst-case scenarios that produce these bounds may never occur. Instead, we use Proposition \ref{prop:wtpofgen} and \ref{prop:wtfairnessgen} to bound $\gamma$ for some feasible matching $M$. 

As before, suppose that policymakers require that $\%F\geq f$, and $\POF \leq p$, for some constants $f$ and $p$. If we know the maximum utility for highly-sensitized patients, we can bound $\gamma$ (for some matching $M$) using the worst-case bounds from Proposition \ref{prop:wtpofgen} and \ref{prop:wtfairnessgen}. Inverting these bounds with $p=\POF$ and $f=\%f$, we have
$$ \gamma \leq \frac{p}{1-p}\left(1+\frac{U_L(M)}{U_H(M)} \right), \hspace{0.2in} \gamma \geq \frac{U_L(M)}{U_H^*}\frac{1}{1-f} - 1.$$
Applying these bounds on $\gamma$ results in the following prioritization set $\mathcal{P}^f_p$, given in \Eqref{eq:wtpsetvarfp}.
\begin{equation}\label{eq:wtpsetvarfp}
\mathcal{P}^f_p = \left\{ \mathbf{\hat{w}} \mid \hat{w}_e = \begin{cases} w_e (1+\gamma) &\text{if } e\in E_H \\ w_e &\text{otherwise} \end{cases}, \frac{U_L(M)}{U_H^*}\frac{1}{1-f} - 1 \leq \gamma \leq \frac{p}{1-p}\left(1+\frac{U_L(M)}{U_H(M)} \right) \right\}
\end{equation}
As before, we might maximize or minimize the prioritization weight over $\mathcal{P}^p_f$ (i.e., weighted fairness), or allow $\gamma$ to vary within the range of $\mathcal{P}^p_f$. Note that allowing $\gamma$ to vary adds variable inequalities, which depends on the decision variables of $M$.